\documentclass[11pt]{article}

\usepackage{multirow}
\usepackage{amssymb,amsmath,amsthm}
\usepackage{color}
\usepackage{graphicx}
  \DeclareGraphicsExtensions{.jpeg,.png,.jpg,.eps,.pdf}
\usepackage{tikz}
\usepackage{epstopdf}
\usepackage{algorithm}
\usepackage{subcaption}
\usepackage{float}
\usepackage{multicol}
\usepackage[noend]{algpseudocode}
\usepackage{comment}

\def\oneb{{\bf 1}}
\def\1{{\bf 1}}

\def\eb{{\bf e}}

\def\N{{\cal N}}
\def\P{{\cal P}}

\def\T{{\cal T}}

\def\Vp{V_\perp}
\def\Vpt{V_\perp^\top}
\def\Ut{U^\top}

\def\R{{\mathbb R}}

\def\Sm{{\mathbb S}}

\def\al{\alpha}
\def\d{\delta}

\def\e{\epsilon}
\def\g{\gamma}
\def\GA{\Gamma}
\def\l{\lambda}

\def\OM{\Omega}

\def\s{\sigma}

\def\th{\theta}

\def\Bb{\bar{B}}
\def\Cb{\bar{C}}

\def\Zb{\bar{Z}}

\def\ap{\rightarrow}

\def\seq{\subseteq}

\def\bi{\{0,1\}}

\def\bino{\bi^{n_r \times n_c}}

\def\imp{\; \Longrightarrow \;}

\def\fa{\; \forall}

\def\drc{(d_r,d_c)}

\def\uh{\hat{u}}

\def\fnrc{\frac{1}{\alpha}}

\def\st{\mbox{ s.t. }}

\def\nm{\Vert}

\renewcommand{\and}{\mbox{$\wedge$}}

\newcommand{\bc}{\begin{center}}
\newcommand{\ec}{\end{center}}
\newcommand{\be}{\begin{equation}}
\newcommand{\ee}{\end{equation}}
\newcommand{\bd}{\begin{displaymath}}
\newcommand{\ed}{\end{displaymath}}
\newcommand{\ba}{\begin{array}}
\newcommand{\ea}{\end{array}}
\newcommand{\ben}{\begin{enumerate}}
\newcommand{\een}{\end{enumerate}}
\newcommand{\bit}{\begin{itemize}}
\newcommand{\eit}{\end{itemize}}
\newcommand{\beq}{\begin{eqnarray}}
\newcommand{\eeq}{\end{eqnarray}}
\newcommand{\btab}{\begin{tabular}}
\newcommand{\etab}{\end{tabular}}
\newcommand{\bfig}{\begin{figure}}
\newcommand{\efig}{\end{figure}}
\newcommand{\btp}{\begin{tikzpicture}}
\newcommand{\etp}{\end{tikzpicture}}
\newcommand{\bcm}{}

\newcommand{\argmin}{\operatornamewithlimits{arg min}}

\renewcommand{\mod}{~{\rm mod}~}

\newcommand{\nmm}[1]{ \nm #1 \nm }
\newcommand{\nmeu}[1]{ \nm #1 \nm_2 }
\newcommand{\nmeusq}[1]{ \nm #1 \nm_2^2 }

\newcommand{\nmF}[1]{ \nm #1 \nm_F }

\newcommand{\nmN}[1]{ \nm #1 \nm_N }

\newcommand{\nmS}[1]{ \nm #1 \nm_S }

\newcommand{\IP}[2]{ \langle #1 , #2 \rangle }
\newcommand{\IPF}[2]{ \langle #1 , #2 \rangle_F }

\newcommand{\rk}{{\rm{rank}}}

\newcommand{\tr}{{\rm{tr}}}

\def\Rno{\R^{n_r \times n_c}}
\def\Rnrc{\R^{n_r \times n_c}}

\def\Xh{\hat{X}}

\def\nmsl1{\nm_{{\rm SL1}}}

\definecolor{verm}{rgb}{0.6,0.2,0.2}
\definecolor{purp}{rgb}{0.3,0.1,0.6}
\definecolor{purple}{rgb}{0.4,0.0,0.6}
\definecolor{bggreen}{rgb}{0.1,0.3,0.1}
\definecolor{dgreen}{rgb}{0.1,0.6,0.1}
\definecolor{black}{rgb}{0.0,0.0,0.0}
\definecolor{crim}{rgb}{0.3,0.1,0.1}
\definecolor{dred}{rgb}{0.5,0.1,0.1}

\definecolor{Blue}{cmyk}{0.65,0.13,0,0}
\definecolor{Black}{cmyk}{0,0,0,1}
\definecolor{Red}{cmyk}{0,1,1,0}
\definecolor{Green}{cmyk}{1,0,1,0}
\definecolor{Orange}{cmyk}{0,0.61,0.87,0.1}
\definecolor{Fuchsia}{cmyk}{0.47,0.91,0,0.08}
\definecolor{PineGreen}{cmyk}{0.92,0,0.59,0.25}

\def\fnrc{\frac{1}{\alpha}}

\def\foal{(1/\al)}

\def\Bb{\bar{B}}
\def\Cb{\bar{C}}
\def\Zb{\bar{Z}}

\def\HOMc{H_{\OM^c}}
\def\PT{\P_\T}
\def\PTP{\P_{\T^\perp}}
\def\Up{U_\perp}
\def\Upt{U_\perp^\top}

\def\W{{\cal W}}

\newtheorem{corollary}{Corollary}{\bf}{\it}
\newtheorem{definition}{Definition}{\bf}{\it}
{\bf}{\rm}
\newtheorem{lemma}{Lemma}{\bf}{\it}
\newtheorem{theorem}{Theorem}{\bf}{\it}
{\bf}{\it}
{\bf}{\it}
{\bf}{\rm}

\begin{document}

\title{
Deterministic Completion of Rectangular Matrices \\
Using Asymmetric Ramanujan Graphs:\\
Exact and Stable Recovery
}

\author{Shantanu Prasad Burnwal and Mathukumalli Vidyasagar
\thanks{The authors are with the Indian Institute of Technology Hyderabad,
Kandi, Telangana 502285, India.
Emails: ee16resch11019@iith.ac.in, m.vidyasagar@iith.ac.in.
This research was supported by the Department of Science and
Technology, and the Science and Engineering Research Board, Government of India.
}
}

\maketitle

\begin{abstract}

In this paper we study the matrix completion problem:
Suppose $X \in \Rno$ is unknown except for a known  upper bound $r$ on its rank.
By measuring a small number $m \ll n_r n_c$ of elements of $X$, is it
possible to recover $X$ exactly with noise-free measurements,
or to construct a good approximation of $X$ with noisy measurements?
Existing solutions to these problems involve sampling the elements uniformly
and at random, and can guarantee exact 
recovery of the unknown matrix only with high probability.
In this paper, we present a \textit{deterministic} sampling method
for matrix completion.
We achieve this by choosing the sampling set
as the edge set of an asymmetric Ramanujan bigraph,
and constrained nuclear norm minimization is the recovery method.
Specifically, we derive sufficient conditions under which the unknown
matrix is completed exactly with noise-free measurements, and is approximately
completed with noisy measurements, which we call ``stable'' completion.
This is in contrast to ``robust'' completion where it is assumed that
measurement errors occur only in a few locations.
We show that \textit{the same}
conditions that suffice for exact completion under noise-free measurements
also suffice to permit an accurate, though not exact, reconstruction
with noisy measurements.

The conditions derived here are only sufficient and more restrictive
than random sampling.
To study how close they are to being necessary, we conducted numerical
simulations on randomly generated low rank matrices, using the LPS families
of Ramanujan graphs.
These simulations demonstrate two facts:
(i) In order to achieve exact completion, 
it appears sufficient
to choose the degree
$d$ of the Ramanujan graph to be $\geq 3r$.
(ii) There is a ``phase transition,'' whereby
the likelihood of success suddenly drops from 100\% to 0\% if the rank
is increased by just one or two beyond a critical value.
The phase transition phenomenon is well-known and well-studied 
in vector recovery using $\ell_1$-norm minimization.
However, it is less studied in
matrix completion and nuclear norm minimization, and not much understood.

\end{abstract}

\section{Introduction}\label{sec:Intro}

\subsection{General Statement}\label{ssec:11}

Compressed sensing refers to the recovery of high-dimensional but
low-complexity objects from a small number of linear measurements.
Recovery of sparse (or nearly sparse) vectors, and recovery
of high-dimensional but low-rank matrices are the two most popular applications of compressed sensing.
The object of study in the present paper is the matrix completion problem,
which is a special case of low-rank matrix recovery.
The matrix completion problem has been receiving attention due to its
application to different areas such as rating systems or recommendation engines
(also known as the ``Netflix problem'' \cite{Candes-Recht08},
sensor localization, structure from motion \cite{Tomasi-Kanade-IJCV92},
and quantum tomography
\cite{Shabani-et-al11,Rodionov-et-al14,KalKosDeu15,Gross-et-al-Phys-Rev10,Li-NIPS11,FGLE12,Xia-EJS17}.
In recommendation engines, the rows represent the products, ranging
from the hundreds to the thousands, while the columns represent customers,
ranging into the millions.
In the community working on recommendation engines,
it is believed, rightly or wrongly,
that each customer uses no more than ten to fifteen features
to make a selection; consequently, the matrix with all entries of
customer preferences has rank no more than fifteen.
In sensor localization, the idea is to infer the pairwise 3-D distance between
$l$ different points, from measuring only a few pairwise distances.
It can be shown that, no matter how large the number $l$ of points is,
this matrix has rank no more than five.
In quantum tomography, the dimensions can be huge, $2^{2^n}$ where $n$
is the number of qubits, but the matrix often has rank one or two.
An excellent survey of the matrix completion problem can be found in
\cite{Davenport-Romberg16}.

\subsection{Problem Definition}\label{ssec:12}

The matrix completion problem can be stated formally as follows:
Suppose $X \in \Rnrc$ is an unknown matrix that 
we wish to recover whose rank is bounded by a known integer $r$.
Let $[n]$ denote the set $\{ 1 , \ldots , n \}$ for each integer $n$.
In the matrix completion problem,
a set $\OM \seq [n_r] \times [n_c]$ is specified, known
as the \textbf{sample set} or \textbf{measurement set}.
To be specific, suppose $\OM = \{ (i_1 , j_1) , \ldots , (i_m,j_m) \}$,
where $|\OM| = m$ is the total number of samples.
We are able to observe the values of the unknown matrix $X$ at the
locations in the set $\OM$, either with or without noise.
In the noise-free case,
the measurements consist of $X_{i,j}$ for all $(i,j) \in \OM$.
Equivalently,
the measured matrix $S$ can be expressed as the Hadamard product\footnote{Recall
that the Hadamard product $C$ of two
matrices $A,B$ of equal dimensions is defined by $c_{ij} = a_{ij} b_{ij}$
for all $i,j$.}
$E_\OM \circ X$ where $E_\OM \in \bi^{n_r \times n_c}$ is defined by
\bd
(E_\OM)_{ij} = \left\{ \ba{ll} 1 & \mbox{if } (i,j) \in \OM , \\
0 & \mbox{if } (i,j) \not\in \OM . \ea \right.
\ed
Note that if $X_{ij}$ happens to equal zero at some sampled
location $(i,j) \in \OM$, then the measured matrix $S := E_\OM \circ X$ will 
also equal zero at the same location.
However, at the unsampled locations, $S = E_\OM \circ X$ will always 
equal zero.
In the \textbf{exact matrix completion} problem, the objective is to
recover the unknown matrix $X$ exactly from $E_\OM \circ X$.
In the case of noisy measurements, the data available to the learner consists
of $E_\OM \circ X + \W$, where $\W$ is the measurement noise.
In this case it is possible to make a distinction between two cases.
In the \textbf{robust matrix completion problem},
it is assumed that the noise matrix $\W$ is \textit{sparse},
that is, $\W_{ij} = 0$ for all but a small number of pairs $(i,j)$.
The objective is still to recover $X$ exactly.
In the \textbf{stable matrix completion} problem,
there are no assumptions on the number of
nonzero elements of $\W$, but there is a known upper bound $\e$ on
the Frobenius norm of $\W$.
In this case, exact recovery of $X$ may not be possible.
Therefore, in stable matrix completion, the aim
is to construct an accurate approximation to $X$.

In the case of noise-free measurements,
one possible approach to the matrix completion problem is to set
\be\label{eq:11}
\Xh = \argmin_{Z \in \Rnrc} \rk(Z) \st E_\OM \circ Z = E_\OM \circ X .
\ee
The above problem is a special case of minimizing the rank of an unknown
matrix subject to linear constraints, and is therefore NP-hard
\cite{RFP10}.
Since the problem is NP-hard, a logical approach is to replace the rank
function by its convex relaxation, which is the \textbf{nuclear norm},
or the sum of the singular values of a matrix, as shown in
\cite{Fazel-Hindi-Boyd01}.
Therefore the convex relaxation of \eqref{eq:11} is
\be\label{eq:12}
\Xh := \argmin_{Z \in \Rnrc} \nmN{Z} \st E_\OM \circ Z = E_\OM \circ X .
\ee
In the case of sparse noisy measurements, the above formulation can
be still used (more in the section on literature review below).
In the case of unstructured noisy measurements,
let us define $X_\W = X + \W$, so that
the noise-corrupted measurement consists of the Hadamard product
$E_\OM \circ X_\W$.
In this case, as suggested in \cite{Candes-Plan-Proc10},
\eqref{eq:12} can be modified to
\be\label{eq:14}
\Xh := \argmin_{Z \in \Rnrc} \nmN{Z} \st \nmF{E_\OM \circ Z - E_\OM \circ X_\W} \leq \e 
\ee
where $\e$ is a known upper bound on the Frobenius norm of the
measurement error matrix $\W$.
For instance, if $\W$ represents Gaussian noise with
known characteristics, a bound on its Frobenius norm can be derived
with a specified prior probability.

\subsection{Contributions of the Present Paper}

In the literature to date, most of the papers assume that the sample set
$\OM$ is chosen at random from $[n_r] \times [n_c]$, either without
replacement as in \cite{Candes-Recht08,Chen-Yudong15}, or with replacement
\cite{Recht-JMLR11}.
The authors are aware of only a handful of papers
\cite{Heiman-et-al13,Bhoj-Jain14,Ash-Wang17,KTT15,AAW18}
in which
a deterministic procedure is suggested for choosing the sample set $\OM$.
In \cite{Heiman-et-al13,Bhoj-Jain14}, the sample set $\OM$ is chosen
as the edge set of a Ramanujan graph.
(This concept is defined below).
Other references such as \cite{KTT15} suggest that the sample set $\OM$
can be viewed as the edge set of a bipartite graph, but do not explicitly
take advantage of this.
In \cite{Bhoj-Jain14}, the authors use nuclear norm minimization
as in \eqref{eq:12}, and claim to derive some sufficient conditions.
However, as shown in the Appendix, there is an error in that paper.
In contrast, we present a correct sufficient condition for exact matrix
completion.
Moreover, our method of proof carries over seamlessly to stable matrix
completion, which is not studied in \cite{Bhoj-Jain14}.
In \cite{Heiman-et-al13}, the approach is max-norm minimization.
The results there are ``universal'' in that they do not involve the
coherence parameter of the unknown matrix $X$ (defined below), and also
apply to the case where the elements are sampled with a \textit{nonuniform}
probability.
The present authors have improved upon those bounds by a factor of two;
see \cite{Shantanu-P1-arxiv19}.
These improvements are achieved through modifying the so-called ``expander
mixing lemma'' for bipartite graphs, which is independent interest.
However, due to space limitations, the max norm minimization approach
to exact matrix completion is not discussed in this paper.

\section{Literature Review}\label{sec:Review}

There are two approaches to choosing the sample set $\OM$,
namely probabilistic and deterministic.
In the probabilistic approach the elements of $\OM$ are chosen at random
from $[n_r] \times [n_c]$, usually, though not always, with a uniform
distribution.
In this setting one can further distinguish between two distinct situations,
namely sampling from $[n_r] \times [n_c]$ \textit{with} replacement or
\textit{without} replacement.
In the deterministic setting, the sample set $\OM$ is interpreted
as the edge set of a bipartite graph.

\subsection{Probabilistic Sampling}\label{ssec:21}

We begin by reviewing the results with probabilistic sampling.
In \cite{Candes-Recht08}, the authors point out that the formulations
\eqref{eq:11} or \eqref{eq:12} do not always recover an unknown matrix.
They illustrate this by taking $X$ as the matrix with a $1$ in the $(1,1)$
position and zeros elsewhere.
In this case, unless $(1,1) \in \OM$, the solution to both \eqref{eq:11} 
and \eqref{eq:12} is the zero matrix, which does not equal $X$.
The difficulty in this case is that the matrix has high ``coherence,'' as
defined next.

\begin{definition}\label{def:coh}
Suppose $X \in \Rnrc$ has rank $r$ and the reduced singular value
decomposition $X = U \Gamma V^\top$, where $U \in \R^{n_r \times r}$,
$V \in \R^{n_c \times r}$, and $\Gamma \in \R^{r \times r}$ is the diagonal
matrix of the nonzero singular values of $X$.
Let $P_U = U U^\top \in \R^{n_r \times n_r}$ denote the orthogonal
projection of $\R^{n_r}$ onto $U \R^{n_r}$.
Finally, let $\eb_i \in \R^{n_r}$ denote the $i$-th canonical basis vector.
Then we define
\be\label{eq:21}
\mu_0(U) := \frac{n_r}{r} \max_{i \in [n_r]} \nmeusq{P_U \eb_i} 
= \frac{n_r}{r} \max_{i \in [n_r]} \nmeusq{u^i} ,
\ee
where $u^i$ is the $i$-th row of $U$.
The quantity $\mu_0(V)$ is defined analogously, and
\be\label{eq:22}
\mu_0(X) := \max \{ \mu_0(U) , \mu_0(V) \} .
\ee
Next, define
\be\label{eq:22a}
\mu_1(X) := \sqrt{ \frac{n_r n_c} {r}} \nmm{U V^\top}_\infty ,
\ee
\end{definition}

It is shown in \cite{Candes-Recht08} that $1\leq \mu_0(U) \leq \frac{n_r}{r}$.
The upper bound is achieved if any canonical basis vector is a column of $U$.
(This is what happens with the matrix with all but one element equalling zero.)
The lower bound is achieved if every element of $U$ has the same magnitude
of $1/\sqrt{n}$, that is, a submatrix of a Walsh-Hadamard matrix.

If one were to sample $m$ out of the $n_r n_c$ elements of the unknown matrix
$X$ \textit{without} replacement, then one is guaranteed that
exactly $m$ distinct elements of $X$ are measured.
This is the approach adopted in \cite{Candes-Recht08}.
However, the disadvantage is that the locations of the
$m$ samples are not independent,
because
once the first element has been selected, there are only $n_r n_c - 1$
choices for the second sample, and so on.
Thus sampling without replacement requires quite advanced probabilistic
analysis.

\begin{theorem}\label{thm:21}
(See \cite[Theorem  1.1]{Candes-Recht08}.)
Suppose there is a known constant $\mu$ such that $\mu \geq
\max \{ \mu_0(X) , \mu_1(X) \}$.
Draw
\be\label{eq:38b}
m \geq Cn^{5/4}r\log(n) 
\ee
samples from $[n_r]\times [n_c]$ without replacement, with a uniform
distribution.
Then with probability atleast $1-\zeta$ where
\be\label{eq:40b}
\zeta = cn^{-3}
\ee
the recovered matrix $\Xh$ using \eqref{eq:12} is the unique solution.
Here $C,c$ are some universal constants that depend on $\mu$,
and $n=\max(n_r,n_c)$.
\end{theorem}

An alternative is to sample the elements of $X$ \textit{with} replacement.
In this case the locations of the $m$ samples are indeed independent.
However, the price to be paid is that, with some small probability,
there would be duplicate samples, so that after $m$ random draws,
the number of elements of $X$ that are measured could be smaller than $m$.
This is the approach adopted in \cite{Recht-JMLR11}.
On balance, the approach of sampling with replacement is easier to analyze.

\begin{theorem}\label{thm:22}
(See \cite[Theorem  2]{Recht-JMLR11}.)
Assume without loss of generality that $n_r \leq n_c$.
Choose some constant $\beta > 1$, and draw
\be\label{eq:36a}
m \geq 32 \max \{ \mu_1^2,\mu_0\} r(n_r+n_c) \beta \log^2(2 n_c) 
\ee
samples from $[n_r] \times [n_c]$ with replacement.
Define $\Xh$ as in \eqref{eq:12}.
Then, with probability at least equal to $1 - \zeta$ where
\be\label{eq:37a}
\zeta = 6 \log(n_c)(n_r+n_c)^{2-2\beta} + n_c^{2 - 2 \sqrt{\beta} } ,
\ee
the true matrix $X$ is the unique solution to
the optimization problem, so that $\Xh = X$.
\end{theorem}

A recent result published in \cite{Chen-Yudong15} gives an improved
sufficient condition that involves only the coherence parameter $\mu_0(X)$,
but not $\mu_1(X)$, as in Theorems \ref{thm:22}.
This approach reduces the number of measurements $m$
by a factor of $r$, but adds a factor of $\log n$.
Therefore, whenever $r > \log n$, the approach of
\cite{Chen-Yudong15} is an improvement.

\begin{theorem}\label{thm:21a}
(See \cite[Theorem 1]{Chen-Yudong15})
Draw 
\be\label{eq:38a}
m \geq C\mu_0rn\log^2(n_r+n_c)
\ee
samples from $[n_r]\times [n_c]$ uniformly at random, with replacement.
Then with probability at least $1-\zeta$ where
\be\label{eq:40a}
\zeta = c_1(n_r+n_c)^{-c_2}
\ee
the recovered matrix $\Xh$ using \eqref{eq:12} is the unique solution.
Here $n=\max(n_r,n_c)$ and $C,c_1,c_2$ are some constants greater than zero.
\end{theorem}

Theorems \ref{thm:21}-\ref{thm:22} gives conditions of exact recovery using
the probabilistic method when there are no measurement errors involved.
For the case of unstructured noisy measurements (that is,
stable matrix completion), the following result is relevant.

\begin{theorem}\label{thm:21n}
(See \cite[Theorem 7]{Candes-Plan-Proc10}.)
Draw samples uniformly from $[n_r]\times [n_c]$ without replacement.
Then with probability atleast $1-\zeta$ where
\be\label{eq:231}
\zeta = cn^{-3}
\ee
the recovered matrix $\Xh$ using \eqref{eq:14} satisfies
\be\label{eq:232}
\nmF{\Xh-X} \leq 
\left[2\sqrt{\frac{C \min(n) }{p}}+1\right]2\e. 
\ee
Here $c,C$ are constants dependent on $\mu$.
\end{theorem}

\subsection{Robust Matrix Completion}\label{ssec:22}

In this subsection we present a \textit{very brief} review of
robust matrix completion.
The interested reader is directed to the references of the papers discussed
here for more information.

In \cite{P-ABN16}, the authors analyze the case of noise-free measurements.
They begin by observing that $X$ has rank $r$, then the columns of $X$
span an $r$-dimensional subspace, call it $\Sm$.
Each measurement leads to some constraints on $\Sm$.
By analyzing the structure of the measurements, it may be possible to
determine the subspace $\Sm$ uniquely.
The following result gives the flavor of the results in
\cite{P-ABN16}.\footnote{For the convenience of the reader, we use the
notation in that paper as opposed to the current notation.}

\begin{theorem}\label{thm:A1}
(See \cite[Theorem 3]{P-ABN16}.
Suppose $X \in \R^{d \times N}$ and has rank $r$ or less.
Suppose $r \leq d/r$, and that each column of $X$ is observed in
at least $l$ rows, distributed uniformly and independently across columns,
where
\be\label{eq:A1}
l \geq \max \left\{ 12 \left( \log \frac{d}{\e} + 1 \right) , 2r \right\} .
\ee
Then with probability at least $1-\e$, the matrix $X$ can be exactly completed.
\end{theorem}

The ``completion algorithm'' is \textit{not} nuclear norm
minimization as in earlier papers.
Rather, one has to solve a set of polynomial equations; see
\cite[Section IV-B]{P-ABN16}.
Thus, while the \textit{sample complexity} of the number of measurements
grows slowly, the \textit{computational complexity} of implementing
the recovery algorithm is higher than with nuclear norm minimization.

Note that there is no measurement noise in the formulation of \cite{P-ABN16}.
Since the results in \cite{P-ABN16} depend on determining which $r$-dimensional
subspaces are compatible with the observations, in principle it is possible
for the recovery algorithm to be robust against \textit{a limited number}
of erroneous measurements.
This is the problem of \textit{robust matix completion}.
One of the recent contributions on this topic is \cite{AAW18}, which also
contains a wealth of references.
The key results are \cite[Theorem 4]{AAW18} and \cite[Theorem 5]{AAW18}.
We do not state those theorems here because of the need to introduce a great 
deal of notation and background material; instead we refer the reader to the
original source.

Another direction of research is \textit{tensor completion} instead of
matrix completion.
At a very basic level, one can think of a tensor as a real number indexed
by more than two integer indices.
However, tensor analysis is far more intricate compared to matrix analysis,
in terms of canonical forms, representation, rank factorization etc.
The notion of ``rank'' extends to tensors.
Therefore it is reasonable to pose the tensor completion problem
as that of reconstructing a tensor from measuring some of its components.
A recent paper \cite{Ash-Wang17} builds on the general approach proposed 
in \cite{P-ABN16} by examining which set of tensors is compatible
with a particular set of measurements, and then enlarging the set of
measurements until only the true but unknown tensor is compatible
with the measurements.

We conclude this brief review by mentioning \cite{KTT15}, which
introduces the notion of reconstructing \textit{a single element}
of the unknown matrix, and then building on that.
One noteworthy feature of this paper is the explicit recognition
of the sample set $\OM$ as the edge set of a bipartite graph.
However, unlike with Ramanujan graphs, the approach does not
take into account the \textit{spectral properties} of this bipartite graph.

\subsection{Alternatives to Nuclear Norm Minimization}

While nuclear norm minimization is the most popular approach to matrix 
completion, there are other approaches, of which just a few are
reviewed here.

In \cite{KMO10a}, the matrix completion problem is solved via an algorithm
called ``OptSpace,'' which incorporates three steps: Trimming, projecting, and
cleaning.
The algorithm can be viewed as optimization on the Grassmanian manifold
of subspaces.
In \cite{KMO10b}, the OptSpace algorithm is extended to encompass measurement
noise.

One approach, inspired by \cite{Burer-Mont03}, is to enforce the constraint
that $\rk(X) \leq r$ by factoring $X$ as $G H^\top$, where
$G \in \R^{n_r \times r}, H \in \R^{n_c \times r}$.
Let $S = E_\OM \circ X$ denote the measured matrix as above.
Then \eqref{eq:11} can be reformulated as
\be\label{eq:16}
\min_{G \in \R^{n_r \times r}, H \in \R^{n_c \times r}}
\nmF{ GH^\top - S}^2 + \l ( \nmF{G}^2 + \nmF{H}^2 ) ,
\ee
where $\nmF{\cdot}$ denotes the Frobenius norm of a matrix, that is,
the $\ell_2$-norm of its vector singular values, and $\l$ is an
adjustable weight.
See \cite[Eq.\ (10)]{Davenport-Romberg16} for more explanation.
The potential difficulty with \eqref{eq:16} is that it is no longer a
convex optimization problem, and thus could have spurious local minima.
Over the years, several papers, including the recent paper \cite{ZLTW-TSP18},
show that there are no spurious local minima.
Moreover, this statement is true for more general problems than \eqref{eq:16},
when the quadratic cost function is replaced by more general cost
functions.

In \cite{YFS-Ent18}, the authors first convert the \textit{constrained}
optimization problem \eqref{eq:12} to the \textit{regularized} problem
\bd
\min_{Z \in \Rno} \nmF{Z - S}^2 + \l \nmm{Z}_1 ,
\ed
where, as before $S = E_\OM \circ X$ is the measured matrix.
Then they replace the term $\nmF{Z - S}^2$ by what they call a
\textbf{correntropy} term, as follows:
\bd
L_\s(Z) := \frac{\s^2}{2} \sum_{(i,j) \in \OM}
[ 1 - \exp( - (Z_{ij} - S_{ij})^2/\s^2) ] ,
\ed
where $\s$ is an adjustable parameter.

In cases where there is more structure to the unknown matrix $X$,
other approaches are possible.
In \cite{Chen-Chi-TIT14}, the authors study the problem of robust compressed
sensing, and formulate it as a structured matrix completion problem.
The problem is then solved via a new algorithm called Enhanced Matrix
Completion.
In \cite{Chen-Chi-TIT15}, the objective is to estimate an unknown
covariance matrix using rank one measurements.
Thus $X$ is assumed to be symmetric and positive definite, which is
then probed via measurements of the form
\be\label{eq:17}
y_l = a_l^\top X a_l + \eta_l = \IPF{X}{a_l a_l^\top} + \eta_l , l \in [m] ,
\ee
where $a_l$ is the probing vector, $\eta_l$ is measurement noise, and $m$
is the number of measurements.
Note that for an arbitrary $X \in \Rno$, $i \in [n_r]$ and $j \in [n_c]$,
we can write
\bd
X_{ij} = \IPF{X}{\eb_i \eb_j^\top} ,
\ed
where $\eb$ denotes an elementary unit vector.
Thus \eqref{eq:17} can be viewed as probing $X$
by a more general rank one matrix.
Earlier, in \cite{Cai-Zhang-AoS15}, it is proposed to use rank one 
measurements of the form
\bd
y_l = \IPF{X}{a_l b_l^\top} , l \in [m] ,
\ed
where $a_l,b_l$ are random Gaussian vectors, and $X$ is a low-rank matrix.
Again, \eqref{eq:17} can be viewed as a specialization of the general
rank one probe to the case of symmetric (and positive semidefinite) matrices.
Note that, in the numerical examples in \cite{Chen-Chi-TIT14}, the
vectors $a_l$ are random Gaussian or Bernoulli vectors.
Finally, in \cite{ZHWC-STSP18}, the unknown matrix $X$ is either a
single-channel or a multi-channel Hankel matrix.

\subsection{Basic Concepts from Graph Theory}\label{ssec:23}

The results presented here make use of the concept of Ramanujan graphs.
Hence we introduce this concept using a bare minimum of graph theory.
Further details about Ramanujan graphs can be found in
\cite{Ram-Murty03,DSV03}.

Suppose a graph has $n$ vertices, so that its adjacency matrix $A$
belongs to $\bi^{n \times n}$.
Such a graph is said to be \textbf{$d$-regular} if every vertex has degree $d$.
Clearly, this is equivalent to saying that every row of $A$ contains
precisely $d$ ones (and also every column, because $A$ is symmetric).
It is easy to show that $\l_1 = d$ is the largest eigenvalue of $A$.
The graph is said to be a \textbf{Ramanujan graph} if the \textit{second
largest} eigenvalue $\l_2$ satisfies the bound
\be\label{eq:15}
| \l_2 | \leq 2 \sqrt{d-1} .
\ee
The significance of the bound in \eqref{eq:15} arises from the so-called
Alon-Boppana bound \cite{Nilli91,Hoory-et-al06}, which states the
following:
If $d$ is kept fixed and $n$ is increased, then the right side of
\eqref{eq:15} is the \textit{best possible bound} for $|\l_2|$.

Suppose $B \in \bino$.
Then $B$ can be interpreted as the biadjacency matrix of a bipartite
graph with $n_r$ vertices on one side and $n_c$ vertices on the other.
If $n_r = n_c$, then the bipartite graph is said to be \textbf{balanced},
and is said to be \textbf{unbalanced} if $n_r \neq n_c$.
The prevailing  convention is to refer to the side with the larger ($n_c$)
vertices as the ``left'' side and the other as the ``right'' side.
A bipartite graph is said to be \textbf{left-regular} with degree $d_c$
if every left vertex has degree $d_c$, and \textbf{right-regular} with
degree $d_r$ if every right vertex has degree $d_r$.
It is said to be \textbf{$(d_r,d_c)$-biregular}
if it is both left- and right-regular with row-degree $d_r$ and
column-degree $d_c$.
Obviously, in this case we must have that $n_r d_r = n_c d_c$.
It is convenient to say that a \textit{matrix} $B \in \bino$ is
``$(d_r,d_c)$-biregular''
to mean that the associated \textit{bipartite graph} is $(d_r,d_c)$-biregular.

If $B$ is the biadjacency matrix of a bipartite graph with $n_r \leq n_c$,
then the full
adjacency matrix of the graph with $n_r + n_c$ vertices looks like
\bd
A = \left[ \ba{cc} 0 & B \\ B^\top & 0 \ea \right] .
\ed
It is easy to verify that the eigenvalues of $A$ are
$\pm \s_1 , \ldots , \pm \s_{n_r}$ together with an appropriate number
of zeros, where $\s_1 , \ldots , \s_{n_r}$ are the singular values of $B$
(some of which could be zero).
Moreover, it is easy to show that $\s_1 = \sqrt{d_r d_c}$.
Now,
The bipartite graph corresponding to $B$ is defined to be
an \textbf{asymmetric Ramanujan graph} if 
\be\label{eq:21a}
|\s_2| \leq \sqrt{d_r-1}+\sqrt{d_c-1}, 
\ee
where $\s_2$ represents second largest singular value of the matrix $B$.
As with the bound in \eqref{eq:15}, it can be shown \cite{Feng-Li96}
that the right side of \eqref{eq:21a} is the \textit{best possible}
bound for $\s_2$.
Throughout this paper we will represent $\s_1$
as the largest singular value and 
$\s_2$ as the second 
largest singular value
of the measurement matrix $E_\OM$.

There are relatively few explicit constructions of Ramanujan graphs.
The present authors have given a new construction in \cite{Shantanu-P2-arxiv19}.
Note that, according to \cite{MSS-FOCS15}, Ramanujan graphs of all degrees
$d$ and all vertex set sizes $n$ exist, though as yet there are no
efficient algorithms for finding them.
See \cite{Cohen16} for some preliminary results in this direction.

\subsection{A Claimed Previous Result}\label{ssec:24}

In this section we present a claimed result from \cite{Bhoj-Jain14}
on matrix completion
using a Ramanujan graph to generate the sampling set.
To facilitate the statement of these results, we reproduce 
two standard coherence assumptions on the unknown matrix $X = U \Gamma V^\top$.
As is standard, we use $A^i, A_j, A_{ij}$ to denote the $i$-th row,
the $j$-th column, and the $(i,j)$-th element of a matrix $A$.
We also use $\nmS{A}$ to denote the spectral or operator
norm of the matrix $A$, that is, the largest singular value of $A$.

\ben
\item[(A1).] There is a known upper bounds $\mu_0(X)$ on $\mu_0(U)$ and
$\mu_0(V)$ respectively.
Hereafter simply write $\mu_0$ for $\mu_0(X)$.
\item[(A2).] There is a constant $\th$ such that
\be\label{eq:23}
\left\nm \sum_{k\in S}\frac{n_r}{d_c} (U^{k\top} U^k) - I_{r} \right\nm_S
\leq \th , \fa S \seq [n_r], |S| = d_c , 
\ee
\be\label{eq:24}
\left\nm \sum_{k\in S}\frac{n_c}{d_r} (V^{k\top} V^k) - I_{r} \right\nm_S
\leq \th , \fa S \seq [n_c],|S| = d_r ,
\ee
where $U^{k\top}$ is shorthand for $(U^k)^\top$,
$V^{k\top}$ is shorthand for $(V^k)^\top$
and $d_r,d_c$ are the degrees of the Ramanujan bigraph $\OM$. 
If Ramanujan graph is a balanced graph, then $n_r=n_c=n$ and $d_r=d_c=d$.
\een

The following result is claimed in \cite{Bhoj-Jain14}.

\begin{theorem}\label{thm:23}
(See \cite[Theorem  4.2]{Bhoj-Jain14}.)
Suppose Assumptions (A1) and (A2) hold.
Choose $E_\OM$ to be the adjacency matrix of $d$ regular graph such that
$\s_2(E_\OM)\leq C\sqrt{d}$, and $\th < 1/6$.
Define $\Xh$ as in \eqref{eq:12}.
With these assumptions, if
\be\label{eq:26}
d\geq 36C^2\mu_0^2r^2 ,
\ee
Then the true matrix $X$ is the unique solution to
the optimization problem \eqref{eq:12}.
In particular, if the graph is a Ramanujan graph, then $C = 2$,
so that \eqref{eq:26} becomes
\be\label{eq:26a}
d \geq 144 \mu_0^2 r^2 .
\ee
\end{theorem}

However, there is one step in the proffered proof of the above theorem
that does not appear to be justified.
More details are given in the Appendix.

\section{New Results}\label{sec:New}

\subsection{Theorem Statement}\label{ssec:31}

In this section we state, without proof, the main result of the paper, and
discuss its implications.
The proof of the theorem makes use of some preliminary results, which
are stated and proved in Section \ref{sec:Prelim}.
The proof of the main theorem is given in Section \ref{sec:Proof}.

To avoid repeating the same text over and over, we introduce the following
\textbf{standing notation}:
Suppose $X\in\Rnrc$ is a matrix of rank $r$ or less, and satisfies the
incoherence assumptions $A1$ and $A2$ with constants $\mu_0$ and
$\th$.\footnote{Note that, unlike \cite{Candes-Recht08,Recht-JMLR11},
we do not require the constant $\mu_1$.}
Suppose $E_\OM\in\bino$ a biadjacency matrix of a $\drc$ biregular graph $\OM$,
and let $\s_1,\s_2$ denote the first and second largest singular value of
the matrix $E_\OM$.
Next, define
\be\label{eq:31a}
\al = \frac{d_c}{n_r} = \frac{d_r}{n_c} = \sqrt{\frac{d_r d_c}{n_r n_c}} ,
\ee
and note that $\al$ is the fraction of elements of $X$ that are being sampled.
Finally, define
\be\label{eq:80aa}
\phi := \frac{\s_2}{\s_1} \mu_0 r ,
\ee
and observe that $\phi$ depends on two distinct pieces of information:
The ratio $\s_2/\s_1$ which depends on the sampling matrix $E_\OM$,
and the quantity $\mu_0$ which depends on the unknown matrix $X$.

\begin{theorem}\label{thm:35}
With the standing notation, suppose without loss of generality that
$n_r \leq n_c$.
Suppose that
$\th + \phi < 1$, and that
\be\label{eq:80bb}
\al > \frac{r(\th^2 + \phi^2)}{(1- (\th+\phi))(1 - \phi^2)} .
\ee
Define constants $c, \g$ as
\be\label{eq:80bc}
c := (1 - \phi) - [ r \al (1 - (\th+\phi)) (\th^2 + \phi^2)]^{1/2} ,
\ee
\be\label{eq:80bd}
\g := 2 \left[ 1 + \frac{n_r}{c^2} \left( 1 + \frac{1}
{\al(1 - (\th+\phi))} \right) \right]^{1/2} ,
\ee
and note that $c > 0$ as a consequence of \eqref{eq:80bb}.
Then every solution $\Xh$ of \eqref{eq:14} satisfies, for every $\d > 0$,
the bound
\be\label{eq:80ac}
\nmF{\Xh-X} \leq (\g+\d) \e .
\ee
\end{theorem}

\begin{corollary}\label{coro:31}
If the hypotheses of Theorem \ref{thm:35} hold, then $\Xh = $ is the 
unique solution of \eqref{eq:12}.
\end{corollary}

\begin{proof}
Substitute $\e = 0$ in \eqref{eq:80ac}.
\end{proof}

Now we compare Theorem \ref{thm:35} with the prior literature discussed
in Section \ref{sec:Review}.

\ben
\item Compared to Theorems \ref{thm:21}, \ref{thm:22} and \ref{thm:21a},
the conclusions of Theorem \ref{thm:35} \textit{always} hold, and not
just with high probability.
\item Unlike Theorems \ref{thm:21} and \ref{thm:22} (but like Theorem
\ref{thm:21a}), our result makes use of only the coherence constant $\mu_0$,
but not $\mu_1$.
\item So far as the authors are aware, the only other result that applies
to the case of \textit{non-sparse} noisy measurements is Theorem \ref{thm:21n}.
That is also a probabilistic result, while Theorem \ref{thm:35} is
deterministic.
\een

\subsection{Sample Complexity Analysis}\label{ssec:32}

Now we attempt to compare the number of measurements required by our
approach with the number required by earlier theorems in Section 
\ref{sec:Review}.
A direct comparison is difficult due to the following factors.
Theorems \ref{thm:21} and \ref{thm:21a} are stated in terms of universal
constants that are not explicitly computable.
Therefore these two theorems can be used only to bound the \textit{rate
of growth} of the number of measurements $m$.
In Theorem \ref{thm:22}, the bound for $m$ is quite explicit.
However, it involves both coherence constants $\mu_0$ and $\mu_1$,
while our bound (as also those in Theorems \ref{thm:21a} and \ref{thm:23})
do not involve the coherence parameter $\th_1$.

Note that, in Theorem \ref{thm:35},
$\al$ is the \textit{fraction} of elements of $X$ that
are measured, so that the total number of measurements is $\al n_r n_c$.
Clearly the smaller the value of $\al$ given by \eqref{eq:80bb}, the
fewer the number of measurements.
Therefore, for a like to like comparison, we compare \eqref{eq:80bb}
with that in \eqref{eq:26a}, even though the proof of Theorem \ref{thm:23}
contains a gap, as shown in the Appendix.

To facilitate the comparison, let us restrict to square matrices,
so that $d_r = d_c = d$, and $n_r = n_c = n$.
Choose the coherence parameter $\mu_0$ as $1.5$, which is only slightly
larger than the theoretical minimum value of $1$.
It is always the case that $\s_1 = d$.
If the graph is a Ramanujan graph, it can be assumed that $\s_2 = 2 \sqrt{d}$
after neglecting the $-1$ term.
Therefore the ratio $\s_2/\s_1$ equals $2/\sqrt{d}$.
Finally let $r = 2$, which is very small (but comparable to the range of
values for which Theorems \ref{thm:21} and \ref{thm:22} apply).
With these numbers, we have that $\mu_0 r  = 3$.
Hence \eqref{eq:26a} gives
\bd
d \geq 144 \times 3^2 = 1,296 .
\ed
Therefore, by applying Theorem \ref{thm:23}, one would have to measure
$1,296$ elements of $X$ in each row and each column.
Clearly the degree $d$ cannot exceed the number of vertices $n$.
Therefore (leaving aside the validity of \eqref{eq:26a}), 
Theorem \ref{thm:23} does not apply unless $n \geq 1,296$.

Now let us apply Theorem \ref{thm:35}.
By definition 
\bd
\phi = (\s_2/\s_1) \mu_0 r = \frac{ 6}{\sqrt{d}}.
\ed
Let us choose $\th = 0.1$.
If we set $d = 800$ which is \textit{lower} than the number from
Theorem \ref{thm:23}, then the lower bound in \eqref{eq:80bb} becomes
\bd
d = 800 \imp \phi = 0.2121 \imp \al > 0.2575 .
\ed
However, if we decrease $d$ to $500$, then we get
\bd
d = 500 \imp \phi = 0.2683 \imp \al > 0.4850 .
\ed
In either case $\al < 1$.
Therefore Theorem \ref{thm:35} is applicable with \textit{fewer measurements
per row and column} than the claimed Theorem \ref{thm:23}.

Note that the size of the matrix $n$ does not enter the bound \eqref{eq:80bb}.
It is easy to see that the total number of measurements is
\bd
m = \max \{ dn , \al n^2 \} ,
\ed
or exactly $\max \{ d, \al n \}$ measurements per row.
So if $n = 2,000$, both numbers of roughly equal.
If the matrix size becomes larger, then $\al$ is the decisive factor
and not $d$.

\section{Preliminary Results}\label{sec:Prelim}

In this section we state and prove various preliminary results
that are used in the proof of Theorem \ref{thm:35}, which
is given Section \ref{sec:Proof}.

The result below is used repeatedly to analyze triple products.

\begin{lemma}\label{lemma:51}
Suppose $M \in \Rno$, $A \in \R^{n_r \times r}$, and $B \in \R^{n_c \times r}$.
Suppose further that $x \in \R^{n_r} , y \in \R^{n_c}$.
Then
\be\label{eq:50a}
x^\top (M \circ (AB^\top))y = \sum_{k \in [r]} (x \circ A_k)^\top M (B_k  \circ y) .
\ee
\end{lemma}

\begin{proof}
The proof follows readily by expanding the triple product.
Note that
\bd
(AB^\top)_{ij} = \sum_{k \in [r]} A_{ik} B_{jk} .
\ed
Therefore
\begin{eqnarray*}
x^\top (M \circ (AB^\top))y & = & 
\sum_{i \in [n_r]} \sum_{j \in [n_c]} x_i
\left( M_{ij} \sum_{k \in [r]} A_{ik} B_{jk} \right) y_j \\
& = & \sum_{k \in [r]} \sum_{i \in [n_r]} \sum_{j \in [n_c]}
x_i A_{ik} M_{ij} B_{jk} y_j  \\
& = & \sum_{k \in [r]} (x \circ A_k)^\top M (B_k  \circ y) ,
\end{eqnarray*}
as desired.
\end{proof}

\begin{lemma}\label{lemma:51a}
Suppose $A \in \R^{n_r \times r}, z \in \R^{n_r}$, and suppose further that
\be\label{eq:51a}
\nmeu{A^i} \leq a , \fa i \in [n_r] .
\ee
Then
\be\label{eq:51b}
\sum_{k \in [r]} \nmeusq{A_k \circ z} \leq a^2 \nmeusq{z} .
\ee
\end{lemma}

\begin{proof}
By definition
\begin{eqnarray*}
\sum_{k \in [r]} \nmeusq{A_k \circ z} & = &
\sum_{k \in [r]} \sum_{l \in [n_r]} A_{lk}^2 z_l^2 \\
& = & \sum_{l \in [n_r]} z_l^2 \left( \sum_{k \in [r]} A_{lk}^2 \right) \\
& \leq & a^2 \nmeusq{z} ,
\end{eqnarray*}
as desired.
\end{proof}

\begin{lemma}\label{lemma:52}
Suppose $M, A, B$ are as in Lemma \ref{lemma:51}.
Suppose further that
\be\label{eq:50b}
\nmeusq{A^i} \leq a^2 , \nmeusq{B^i} \leq b^2 .
\ee
Then
\be\label{eq:50c}
\nmS{M\circ (AB^\top)} \leq ab \nmS{M} .
\ee
\end{lemma}

\begin{proof}
Recall that, for any matrix $X$, we have that
\bd
\nmS{X} = \max_{\nmeu{x} = 1 , \nmeu{y} = 1} x^\top X y .
\ed
In particular
\begin{eqnarray*}
\nmS{M\circ (AB^\top)} & = & \max_{\nmeu{x} = 1 , \nmeu{y} = 1}
x^\top (M \circ (AB^\top))y \\
& = & \max_{\nmeu{x} = 1 , \nmeu{y} = 1} 
\sum_{k \in [r]} (x \circ A_k)^\top M (B_k  \circ y) ,
\end{eqnarray*}
where the last step follows from Lemma \ref{lemma:51}.
Now fix $x, y$ such that $\nmeu{x} = 1 , \nmeu{y} = 1$.
Then
\be\label{eq:50e}
x^\top (M \circ (AB^\top))y \leq \nmS{M}
\sum_{k \in [r]} \nmeu{x \circ A_k} \nmeu{B_k  \circ y} .
\ee
%
Next, apply Schwarz's inequality to deduce that
\beq
\sum_{k \in [r]} \nmeu{x \circ A_k} \nmeu{B_k  \circ y} & \leq &
\left( \sum_{k \in [r]} \nmeusq{x \circ A_k} \right)^{1/2} \nonumber \\
& \cdot & \left( \sum_{k \in [r]} \nmeusq{B_k  \circ  y} \right)^{1/2} \nonumber
\\
& \leq & ab , \label{eq:50f}
\eeq
where the last step follows from applying Lemma \ref{lemma:51a} twice,
and observing that $\nmeu{x} = 1 , \nmeu{y} = 1$.
Substituting this bound in \eqref{eq:50e} shows that
\bd
x^\top (M \circ (AB^\top))y \leq ab \nmS{M} 
\ed
whenever $\nmeu{x} = 1 , \nmeu{y} = 1$, which is the desired result.
\end{proof}

Now we introduce a special matrix $M$, which plays a central role in
solving the matrix completion problem.

\begin{lemma}\label{lemma:51b}
Define
\be\label{eq:39a}
M := \foal E_\OM - \oneb_{n_r \times n_c} .
\ee
Then $\nmS{M} = \s_2/\al$.
\end{lemma}

\begin{proof}
Because $E_\OM$ is biregular, it follows that $\s_1 = \sqrt{d_rd_c}$
is the largest singular value of $E_\OM$.
Moreover, $(1/\sqrt{n_r}) \oneb_{n_r}^\top$
is a row singular vector of $E_\OM$, and $(1/\sqrt{n_c}) \oneb_{n_c}$
is a column singular vector.
Therefore a SVD of $E_\OM$ looks like
\bd
E_\OM = \sqrt{ \frac{d_r d_c} {n_r n_c} } \oneb_{n_r \times n_c} + B
= \al \oneb_{n_r \times n_c} + B ,
\ed
where the largest singular value of $B$ is $\s_2$.
The desired bound follows by observing that $M = \foal B$.
\end{proof}

\begin{theorem}\label{thm:32}
Subject to the standing notation, we have that
\be\label{eq:39b}
\nmS{M \circ X} \leq \phi \nmS{X} ,
\ee
where $\phi$ is defined in \eqref{eq:80aa}, and
$\nmS{\cdot}$ denotes the spectral norm (largest singular value)
of a matrix.
\end{theorem}

\textbf{Remark:}
\bit
\item Note that Lemma \ref{lemma:51} is essentially the same as
\cite[Theorem 4.1]{Bhoj-Jain14}.
However, our proof is much simpler.
\item Observe that $\oneb_{n_r \times n_c} \circ X = X$ for all $X$.
Therefore
\bd
M \circ X = \fnrc E_\OM \circ X - X .
\ed
Therefore Theorem \ref{thm:32} states that the sampled matrix $E_\OM \circ X$,
normalized by the factor $1/\al$ for the fraction of elements sampled,
can provide an approximation to $X$.
The smaller the constant $\phi$, the better the approximation.
\eit

\begin{proof}
Now suppose $X = U \GA V^\top$ is a singular value decomposition of $X$,
where $\GA = {\rm Diag}(\g_1 , \ldots , \g_r)$.
Define $A = U \GA , B = V$.
Then $X = A B^\top$.
Moreover
\begin{eqnarray*}
\sum_{k \in [r]} A_{ik}^2 & = & \sum_{k \in [r]} U_{ik}^2 \g_k^2 
\leq \g_1^2 \sum_{k \in [r]} U_{ik}^2 \\
& \leq & \nmS{X}^2 \frac{\mu_0 r}{n_r} ,
\end{eqnarray*}
because $\nmS{X} = \g_1$, and the definition of the coherence $\mu_0$.
Similarly
\bd
\sum_{k \in [r]} B_{ik}^2 = \sum_{k \in [r]} B_{ik}^2 \leq \frac{\mu_0 r}{n_c} .
\ed
Now apply Lemma \ref{lemma:52} with
\bd
a = \nmS{X} \sqrt{ \frac{\mu_0 r}{n_r} } ,
b = \sqrt{ \frac{\mu_0 r}{n_c} } ,
\ed
and note that $\al \sqrt{n_r n_c} = \sqrt{d_r d_c} = \s_1$.
Then \eqref{eq:50c} becomes
\bd
\nmS{ M \circ X} \leq \frac{\s_2}{\s_1} \mu_0 r \nmS{X} = \phi \nmS{X} ,
\ed
as desired.
\end{proof}

\begin{lemma}\label{lemma:81}
Suppose $E_\OM \in \bino$ is a $\drc$-biregular sampling matrix, and
let $U, V, \th$ be as before.
\ben
\item
For arbitrary $B \in \R^{n_c \times r}$, define
\be\label{eq:811}
F^\top = U^\top M \circ (U B^\top ) .
\ee
Then
\be\label{eq:812a}
\nmeu{F^i} \leq \th \nmeu{B^i} , \fa i \in [n_c] ,
\ee
\be\label{eq:812}
\nmF{F} \leq \th \nmF{B} .
\ee
\item
For arbitrary $C \in \R^{n_r \times r}$, define
\be\label{eq:813}
G = M \circ (C V^\top ) V .
\ee
Then
\be\label{eq:814a}
\nmeu{G_j} \leq \th \nmeu{C_j} , \fa j \in [r] ,
\ee
\be\label{eq:814}
\nmF{G} \leq \th \nmF{C} .
\ee
\een
\end{lemma}

\begin{proof}
We begin by proving \eqref{eq:812a}.
Equation \eqref{eq:812} is a direct consequence of \eqref{eq:812a}

In view of the definition of the matrix $M$, \eqref{eq:811} can be rewritten as
\bd
F^\top := \foal U^\top E_\OM \circ (UB^\top) - B^\top .
\ed
Fix $i \in [r], j \in [n_c]$.
Then
\begin{eqnarray*}
F_{ji} & = & (F^\top)_{ij} = \eb_i^\top F^\top \eb_j \\
& = & \foal \eb_i^\top U^\top E_\OM \circ (UB^\top) \eb_j - B_{ji} .
\end{eqnarray*}
Let us focus on the first term after ignoring the factor of $1/\al$.
From Lemma \ref{lemma:51}, specifically \eqref{eq:50a}, we get
\begin{eqnarray*}
\eb_i^\top U^\top E_\OM \circ (UB^\top) \eb_j & = &
U_i^\top E_\OM \circ (UB^\top) \eb_j \\
& = & \sum_{k \in [r]} (U_i  \circ  U_k)^\top E_\OM (B_k  \circ  \eb_j ) .
\end{eqnarray*}
Now observe that $B_k  \circ  \eb_j = B_{jk} \eb_j$, so that
$E_\OM (B_k  \circ  \eb_j ) = (E_\OM)_j B_{jk} \eb_j$.
Therefore
\bd
\eb_i^\top U^\top E_\OM \circ (UB^\top) \eb_j =
\sum_{k \in [r]} (U_i  \circ  U_k)^\top (E_\OM)_j B_{jk} .
\ed
For this fixed $j$, define
\bd
\N(j) = \{ l \in [n_r] : (E_\OM)_{lj} = 1 \} ,
\ed
and note that $|\N(j)| = d_c$ due to regularity.
Then, for fixed $k \in [r]$, we have
\begin{eqnarray*}
(U_i  \circ  U_k)^\top (E_\OM)_j & = & \sum_{l \in \N(j)} U_{li} U_{lk} \\
& = & \left[ \sum_{l \in \N(j)} U^{l\top} U^{k\top} \right]_{ik} .
\end{eqnarray*}
Therefore
\begin{eqnarray*}
(F^\top)_{ij} & = & \foal \eb_i^\top U^\top E_\OM \circ (UB^\top) \eb_j - B_{ji} \\
& = & \foal \sum_{k \in [r]} 
\left[ \sum_{l \in \N(j)} U^{l\top} U^{k\top} \right]_{ik} B_{jk} - B_{ji} \\
& = & \left( \left[ \foal \sum_{l \in \N(j)} U^{l\top} U^l - I_r \right]
B^\top \right)_{ij} .
\end{eqnarray*}
By \eqref{eq:23}, the matrix inside the square brackets has spectral norm
$\leq \th$.
Therefore
\bd
\nmeu{(F^\top)_j} \leq \th \nmeu{(B^\top)_j} , \fa j \in [n_c] ,
\ed
which is \eqref{eq:812a}.
Taking the norm squared and summing over all $j$ proves \eqref{eq:812},
after noting that a matrix and its transpose have the same Frobenius norm.
This establishes Item (1).

To prove Item (2), we use Item (1).
Note that $(X \circ Y)^\top = X^\top  \circ  Y^\top$.
So \eqref{eq:813} is equivalent to
\bd
G^\top = \foal V^\top E_\OM^\top \circ (V C^\top) - C^\top .
\ed
Now every column of $E_\OM^\top$ (or every row of $E_\OM$) contains $d_r$
ones.
Therefore \eqref{eq:814a} follows from \eqref{eq:812a}, and
\eqref{eq:814} follows from \eqref{eq:812}.
\end{proof}

Next, define $\T \seq \Rno$ to be the subspace spanned by 
all matrices of the form $UB^\top $ and $CV^\top $.
It is easy to show that the projection operator $\P_\T $ equals
\begin{eqnarray*}
\P_\T  Z & = & U U^\top Z + Z V V^\top - U U^\top Z V V^\top \\
& = & UU^\top Z+ \Up \Upt Z V V^\top \\
& = & U U^\top Z V_\perp V_\perp^\top + Z V V^\top ,
\end{eqnarray*}
where $\Up, \Vp$ are chosen such that such that $[ U \; \; \Up  ]$
and $[ V \; \; \Vp ]$ are square and orthogonal.
This ensures that
$\Up \Upt = I_{n_r} - UU^\top$ and
$V_\perp V_\perp^\top = I_{n_c} - V V^\top$.

The next two lemmas characterize the map $\T \ap \T$ defined by
$Z \mapsto (1/\alpha) P_\T E_\OM\circ Z - Z$.
Note that this is not quite the map $M \circ Z$, because of the presence
of the term $\PT$.

\begin{lemma}\label{lemma:82}
Suppose $E_\OM \in \bino$ is a $\drc$-biregular sampling matrix,
and that $Z \in \T$.
Define
\be\label{eq:821}
B^\top = U^\top Z , C = \Up \Upt ZV ,
\ee
so that $Z = UB^\top + CV^\top$.
Next, define
\be\label{eq:822}
\Zb := (1/\alpha) P_\T E_\OM\circ Z - Z ,
\ee
\be\label{eq:823}
\Bb^\top = U^\top \Zb , \Cb = \Up \Up^\top \Zb V .
\ee
Then
\be\label{eq:824}
\nmF{\Bb} \leq \th \nmF{B} + \phi \nmF{C} ,
\ee
\be\label{eq:825}
\nmF{\Cb} \leq \phi \nmF{B} + \th \nmF{C} .
\ee
\end{lemma}

\textbf{Remark:}
The above two relations can be expressed compactly as
\be\label{eq:825a}
\left[ \ba{c} \nmF{\Bb} \\ \nmF{\Cb} \ea \right] \leq
\left[ \ba{cc} \th & \phi \\ \phi & \th \ea \right]
\left[ \ba{c} \nmF{B} \\ \nmF{C} \ea \right] .
\ee

\begin{proof}
We establish \eqref{eq:824}, and the proof of \eqref{eq:825} is entirely
similar.

The definition of $\P_\T$ makes it clear that
\bd
U^\top \P_\T Y = U^\top Y , \Up \Upt \P_\T Y = \Up \Upt Y , 
\fa Y \in \Rno .
\ed
Therefore
\begin{eqnarray*}
\Bb^\top & = &  U^\top ( \foal E_\OM \circ ( UB^\top) - UB^\top ) \\
&   & + U^\top \foal  E_\OM \circ (CV^\top) - U^\top C V^\top \\
& = & U^\top M \circ (U B^\top) + U^\top M \circ ( C V^\top ) .
\end{eqnarray*}
Define $\Bb^\top = \Bb^\top_1 + \Bb^\top_2$, where
\bd
\Bb_1^\top = U^\top M \circ (UB^\top) , \Bb_2^\top = U^\top M \circ (CV^\top) .
\ed
Then it follows from Lemma \ref{lemma:81} that
\be\label{eq:827}
\nmF{\Bb_1} \leq \th \nmF{B} .
\ee

To estimate $\nmF{\Bb_2} = \nmF{\Bb_2^\top}$, we proceed as follows:
\bd
(\Bb_2^\top)^i = \eb_i^\top \Bb_2^\top = \eb_i^\top U^\top M \circ (CV^\top)
= U_i^\top M \circ (CV^\top) .
\ed
\begin{eqnarray*}
\nmeusq{ (\Bb_2^\top)^i } & = & \max_{y \in \R^{n_c} , \nmeu{y} = 1}
(\Bb_2^\top)^i y \\
& = & \max_{ \nmeu{y} = 1 } U_i^\top M \circ (CV^\top) y .
\end{eqnarray*}
Fix a $y \in \R^{n_c}$ such that $\nmeu{y} = 1$ but otherwise arbitrary, 
and define
\bd
\psi_i = U_i^\top M \circ (CV^\top) y .
\ed
Then it follows by Lemma \ref{lemma:51} that
\beq
\psi_i & = & \sum_{k \in [r]} (U_i \circ C_k)^\top M (V_k \circ y) \nonumber \\
& \leq & \frac{\s_2}{\al} \sum_{k \in [r]}  \nmeu{U_i  \circ  C_k} \nmeu{V_k  \circ  y} 
\nonumber \\
& \leq & \frac{\s_2}{\al}
\left( \sum_{k \in [r]} \nmeusq{U_i  \circ  C_k} \right)^{1/2} \nonumber \\
& \cdot &
\left( \sum_{k \in [r]} \nmeusq{V_k  \circ  y}  \right)^{1/2} , \label{eq:828}
\eeq
where we use Schwarz's inequality in the last step.
Now observe that, by the definition of the coherence $\mu_0$, we have
\bd
\nmeusq{V^k} \leq \frac{\mu_0 r}{n_c} , \fa k \in [r] .
\ed
Therefore it follows from Lemma \ref{lemma:51a} that
\bd
\sum_{k \in [r]} \nmeusq{V_k  \circ  y} \leq \frac{\mu_0 r}{n_c}
\ed
because $\nmeusq{y} = 1$.
Next, from Lemma \ref{lemma:51a} and
\bd
\nmeusq{U_k} \leq \frac{\mu_0 r}{n_r} , \fa k \in [r] ,
\ed
it follows that
\begin{eqnarray*}
\sum_{i \in [r]} \sum_{k \in [r]} \nmeusq{U_i \circ C_k} 
& = & \sum_{k \in [r]} \sum_{i \in [r]} \nmeusq{U_i \circ C_k} \\
& \leq & \frac{\mu_0 r}{n_r} \sum_{k \in [r]} \nmeusq{C_k} \\
& = & \frac{\mu_0 r}{n_r} \nmF{C}^2 .
\end{eqnarray*}
Combining both bounds gives
\bd
\nmF{\Bb_2}^2 \leq \frac{ \s_2^2 }{\al^2} \frac{(\mu_0 r)^2}{n_r n_c} \nmF{C}^2
= \left( \frac{\s_2}{\s_1} \mu_0 r \right)^2 \nmF{C}^2  = \phi^2 \nmF{C}^2 ,
\ed
after noting that $\al \sqrt{n_r n_c} = \s_1$.
Taking square roots of both sides gives
\bd
\nmF{\Bb_2} \leq \phi \nmF{C} .
\ed
Finally
\bd
\nmF{\Bb} \leq \nmF{\Bb_1} + \nmF{\Bb_2} 
\leq \th \nmF{B} + \phi \nmF{C} .
\ed
which is \eqref{eq:824}.
The proof of \eqref{eq:825} is entirely similar.
\end{proof}

\begin{lemma}\label{lemma:83}
Suppose $E_\OM \in \bino$ is a $\drc$-biregular sampling matrix,
that $Z \in \T$,
and as in \eqref{eq:822}, define
\be\label{eq:822a}
\Zb := (1/\alpha) P_\T E_\OM \circ Z - Z ,
\ee
Then
\be\label{eq:830}
\nmF{\Zb} \leq (\th + \phi) \nmF{Z} ,
\ee
where $\th, \phi$ are defined in \eqref{eq:824} and \eqref{eq:825}
respectively.
\end{lemma}

\textbf{Remark:}
The above lemma can be stated as follows:
The map $Z \mapsto (1/\alpha) P_\T E_\OM \circ Z - Z$, when restricted to $\T$,
has an operator norm $\leq \th + \phi$.

\begin{proof}
Define, as before,
\bd
B^\top = U^\top Z , C = \Up \Upt ZV ,
\ed
\bd
\Bb^\top = U^\top \Zb , \Cb = \Up \Up^\top \Zb V ,
\ed
so that $Z = UB^\top + CV^\top$, $\Zb = U \Bb^\top + \Cb V^\top$.
Note that
\bd
\IPF{UB^\top}{CV^\top} = \tr( B U^\top C V^\top ) = 0 ,
\ed
because $U^\top C = 0$.
Therefore
\begin{eqnarray*}
\nmF{Z}^2 & = & \nmF{UB^\top}^2 + \nmF{CV^\top}^2 + 2 \IPF{UB^\top}{CV^\top} \\
& = & \nmF{UB^\top}^2 + \nmF{CV^\top}^2 = \nmF{B}^2 + \nmF{C}^2 ,
\end{eqnarray*}
because left multiplication by $U$ and right multiplication by $V^\top$
preserve the Frobenius norm.
Similarly
\bd
\nmF{\Zb}^2 = \nmF{\Bb}^2 + \nmF{\Cb}^2 .
\ed
Now it is easy to verify that the spectral norm of the matrix in \eqref{eq:825a}
is $\th + \phi$.
Therefore
\begin{eqnarray*}
\nmF{\Zb}^2 & = & \nmF{\Bb}^2 + \nmF{\Cb}^2 \leq (\th + \phi)^2
( \nmF{B}^2 + \nmF{C}^2 ) \\
& = & (\th + \phi)^2 \nmF{Z}^2 .
\end{eqnarray*}
This is the desired conclusion.
\end{proof}

\section{Proof of Theorem \ref{thm:35}}\label{sec:Proof}

The proof of Theorem \ref{thm:35} depends on the following auxiliary lemma,
which is reminiscent of the ``dual certificate'' condition that is widely
used in solving the matrix completion problem; see for example
\cite[Theorem 2]{Recht-JMLR11}.
It should be noted that Lemma \ref{lemma:31} provides a framework for
a solution, which we then specialize, somewhat inefficiently,
in Theorem \ref{thm:35}.
Finding better ways to apply Lemma \ref{lemma:31} is a problem for
future research.

Let $X = U \Gamma V^\top$ be a reduced SVD of $X$.
Throughout, we use the symbols $\T, \PT, \PTP, \Up, \Vp$
introduced in the previous section.

\begin{lemma}\label{lemma:31}
Suppose there exists a matrix $Y \in \Rno$ such that
$E_\OM \circ Y = Y$, and constants $k_1 , k_2 \in (0,1), k_3 > 0$ such that
\be\label{eq:105}
\nmS{\P_{\T^\perp}(Y)}\leq k_1 ,
\ee
\be\label{eq:105a}
\left\nm {\foal\P_\T  E_\OM \circ Z - Z} \right\nm_F
\leq k_2 \nmF{Z} , \fa Z \in \T ,
\ee
\be\label{eq:106}
\nmF{UV^T-\P_\T (Y)}\leq k_3 ,
\ee
and
\be\label{eq:106a}
k_3 < (1-k_1) (\al (1-k_2))^{1/2} .
\ee
Define the constant
\be\label{eq:106b}
c := (1-k_1) - (\al (1-k_2))^{-1/2} k_3 .
\ee
Then every solution $\Xh$ of \eqref{eq:14} satisfies the bound
\be\label{eq:106c}
\nmF{\Xh - X} \leq 2 \left[ 1 + \frac{n_r}{c^2}
\left( 1 + \frac{1}{\al(1-k_2)} \right) \right]^{1/2} \e .
\ee
\end{lemma}

\begin{proof}
(Of Lemma \ref{lemma:31}.)
Observe that
\bd
\IPF{X}{UV^\top} = \nmN{X} .
\ed
Define $H = \Xh - X$, so that $\Xh = X + H$.
Also, it can be assumed without loss of generality that $\W$ is
supported on $\OM$, so that $E_\OM \circ \W = \W$ and $E_{\OM^c} \circ \W = 0$.
Now we can write
\beq\label{eq:32}
\nmF{E_\OM\circ H}
& = &\nmF{E_\OM\circ(\Xh-X-\W)+E_\OM\circ\W} \nonumber\\
& \leq & \nmF{E_\OM\circ(\Xh-X-\W)} + \nmF{E_\OM\circ\W} \nonumber\\
& \leq & 2\e ,
\eeq
because (i) $\Xh$ is feasible for \eqref{eq:14}, and thus
$\nmF{E_\OM\circ(\Xh-X-\W)} \leq \e$ (see \eqref{eq:14}),
and $\nmF{W} \leq \e$.
Next,
\beq\label{eq:33}
\nmF{H}^2
& = & \nmF{E_\OM\circ H}^2 + \nmF{E_{\OM^c}\circ H}^2 \nonumber\\
& \leq & 4\e^2 + \nmF{E_{\OM^c}\circ H}^2 .
\eeq
Therefore, once we are able to find an upper bound for $\nmF{E_{\OM^c}\circ H}$,
the above relationship leads to an upper bound for $\nmF{H} = \nmF{\Xh-X}$.

Define $H_\OM = E_\OM\circ H$ and $H_{\OM^c} = E_{\OM^c}\circ H$.
Then
\beq\label{eq:34}
\nmN{\Xh}
& = & \nmN{X + H} \nonumber\\
& = & \nmN{X + H_{\OM^c} + H_\OM } \nonumber\\
& \geq & \nmN{X + H_{\OM^c}} - \nmN{H_\OM}
\eeq
Next, write $\HOMc = \PT(\HOMc) + \PTP(\HOMc)$, and note that
$\IPF{Y}{\HOMc} = 0$ because $Y$ is supported on $\OM$ and $\HOMc$ is
supported on $\OM^c$.
Now observe that, for any matrix $B \in \Rno$, we have that
\be\label{eq:34a}
\nmN{B} = \max_{\nmS{A} \leq 1} \IPF{A}{B},
| \IPF{A}{B} | \leq \nmS{A} \nmN{B} .
\ee
Therefore
\begin{align}\label{eq:35}
&\nmN{X + H_{\OM^c}} \nonumber \\
& \geq  \IPF{UV^T+\Up\Vpt}{X + H_{\OM^c}} \nonumber \\
& =^{(a)}  \nmN{X} + \IPF{UV^T+\Up\Vpt}{H_{\OM^c}}
        - \IPF{Y}{H_{\OM^c}} \nonumber \\
& =  \nmN{X} + \IPF{UV^T-\P_\T (Y)}{\P_\T(H_{\OM^c})} \nonumber \\
&   \ \ + \IPF{\Up\Vpt-\P_{\T^\perp}(Y)}{\P_{\T^\perp}(H_{\OM^c})}
\nonumber \\
&\geq^{(b)}  \nmN{X} - \nmF{UV^T-\P_\T (Y)}\nmF{\P_\T(H_{\OM^c})}
\nonumber \\
&  \ + (1-\nmS{\P_{\T^\perp}(Y)})\nmN{\P_{\T^\perp}(H_{\OM^c})}
\nonumber \\
\end{align}
where $(a)$ follows from $\IPF{Y}{\HOMc} = 0$, and $(b)$ follows from
\eqref{eq:34a}.
Using \eqref{eq:34} and \eqref{eq:35} together we get
\beq\label{eq:36}
\nmN{\Xh} & \geq & \nmN{X} - \nmF{UV^T-\P_\T (Y)}\nmF{\P_\T(H_{\OM^c})}
\nonumber \\
& + & (1-\nmS{\P_{\T^\perp}(Y)})\nmN{\P_{\T^\perp}(H_{\OM^c})}
\nonumber \\
& - & \nmN{H_\OM}.
\eeq
On the other hand, $\nmN{\Xh} \leq \nmN{X}$ because $X$ is feasible for
the problem in \eqref{eq:14}, and $\Xh$ is a solution of \eqref{eq:14}.
Substituting this into \eqref{eq:36}, cancelling $\nmN{X}$ from
both sides, and rearranging terms, gives
\beq\label{eq:37}
 \nmN{H_\OM} & \geq &
 (1-\nmS{\P_{\T^\perp}(Y)})\nmN{\P_{\T^\perp}(H_{\OM^c})}
 \nonumber \\
 & & \ - \nmF{UV^T-\P_\T (Y)}\nmF{\P_\T(H_{\OM^c})}
\eeq
Now,
\begin{align}\label{eq:38}
& \nmF{E_\OM\circ \P_\T(H_{\OM^c})}^2 \nonumber \\
& =  \IPF{E_\OM\circ \P_\T(H_{\OM^c})}{\P_\T(H_{\OM^c})}
\nonumber \\
& =  \IPF{\P_\T E_\OM\circ\P_\T(H_{\OM^c})-\al \P_\T(H_{\OM^c})}{\P_\T(H_{\OM^c})} \nonumber \\
& \ + \al \IPF{\P_\T(H_{\OM^c})}{\P_\T(H_{\OM^c})} \nonumber \\
& \geq^{(a)} \al \nmF{\P_\T(H_{\OM^c})}^2 \nonumber \\
& \ - \nmF{ \PT (E_\OM \circ \PT (\HOMc) - \al \PT(\HOMc))} \nmF{\P_\T(H_{\OM^c})} \nonumber \\
& \geq^{(b)} \al \nmF{\P_\T(H_{\OM^c})}^2
- \al k_2 \nmF{\P_\T(H_{\OM^c})}^2 \nonumber \\
& = \al (1 - k_2) \nmF{\P_\T(H_{\OM^c})}^2 ,
\end{align}
where $(a)$ follows from $\PT^2 = \PT$, and $(b)$ follows from
assumption \eqref{eq:105a}.
Next, note that
$E_\OM \circ \HOMc = 0 $, which in turn implies that
\bd
E_\OM \circ \PT \HOMc = - E_\OM \circ \PTP \HOMc ,
\ed
so that
\bd
\nmF{E_\OM \PT \HOMc } = \nmF{E_\OM \PTP \HOMc } ,
\ed
Therefore
\begin{align}\label{eq:39}
\nmN{\P_{\T^\perp}(H_{\OM^c})}
& \geq \nmF{\P_{\T^\perp}(H_{\OM^c})} \nonumber\\
& \geq
\nmF{E_\OM\circ \P_{\T^\perp}(H_{\OM^c})} \nonumber \\
& = \nmF{E_\OM\circ \P_\T(H_{\OM^c})} \nonumber \\
& \geq (\al(1 - k_2))^{1/2} \nmF{\P_\T(H_{\OM^c})}
\end{align}
where the last step follows from \eqref{eq:38}.
Now using \eqref{eq:39} in \eqref{eq:37} gives us
\begin{align}\label{eq:40}
\nmN{H_\OM} & \geq
(1-\nmS{\P_{\T^\perp}(Y)})\nmN{\P_{\T^\perp}(H_{\OM^c})}
\nonumber \\
&  \ - (\al (1-k_2))^{-1/2} \nmF{UV^T-\P_\T (Y)}\nmN{\P_{\T^\perp}(H_{\OM^c})}\nonumber\\
& \geq [ (1 - k_1) - (\al(1-k_2))^{-1/2} k_3 ] \nmN{\P_{\T^\perp}(H_{\OM^c})}\nonumber\\
& = c \nmN{\P_{\T^\perp}(H_{\OM^c})} ,
\end{align}
where we use the assumptions \eqref{eq:105} and \eqref{eq:106a}, together 
with \eqref{eq:39}.
Next
\begin{align}\label{eq:41}
\nmF{\P_{\T^\perp}(H_{\OM^c})}
& \leq \nmN{\P_{\T^\perp}(H_{\OM^c})}
\leq^{(a)} (1/c) \nmN{H_\OM} \nonumber\\
& \leq^{(b)} (\sqrt{n_r}/c) \nmF{H_\OM}
\leq^{(c)} 2 (\sqrt{n_r}/c) \e  ,
\end{align}
where $(a)$ follows from \eqref{eq:40},
$(b)$ follows from the fact that $H_\OM \in \Rno$
and thus has rank no more than $\min \{ n_r , n_c \} = n_r$, and
$(c)$ follows from \eqref{eq:32}.
Using \eqref{eq:38} we get
\begin{align}\label{eq:42}
\nmF{\P_\T(H_{\OM^c})}
& \leq (\al(1-k_2))^{-1/2}\nmF{E_\OM\circ \P_\T(H_{\OM^c})}
\nonumber\\
& =^{(a)} (\al(1-k_2))^{-1/2}\nmF{E_\OM\circ \P_{\T^\perp}(H_{\OM^c})}
\nonumber\\
& \leq (\al(1-k_2))^{-1/2}\nmF{\P_{\T^\perp}(H_{\OM^c})}
\nonumber\\
\end{align}
where, as above, $(a)$ follows from the fact $E_\OM \circ H_{\OM^c}=0$.
Now \eqref{eq:33} can be written as
\begin{align}\label{eq:43}
\nmF{H}^2
& \leq  4\e^2 + \nmF{\HOMc}^2 \nonumber\\
& = 4\e^2 + \nmF{\P_\T (\HOMc)}^2 + \nmF{\PTP (\HOMc)}^2 \nonumber\\
	& \leq^{(a)} 4 \e^2 + \left( 1 + \frac{1}{\al(1-k_2)} \right)
	\nmF{\PTP(\HOMc)}^2 \nonumber \\
	& \leq^{(b)} 4 \e^2 + \left( 1 + \frac{1}{\al(1-k_2)} \right)
	\frac{4 n_r}{c^2} \e^2 ,
\end{align}
where $(a)$ follows from \eqref{eq:42} and $(b)$ follows from \eqref{eq:41}.
It now follows that
\bd
\nmF{H} \leq 2 \left[ 1 + \frac{n_r}{c^2}
\left( 1 + \frac{1}{\al(1-k_2)} \right) \right]^{1/2} \e ,
\ed
which is \eqref{eq:106c}.
\end{proof}

\begin{proof}
(Of Theorem \ref{thm:35})
At last we come to the proof of the main theorem.
Suppose $\th + \phi < 1$ (which automatically implies that $\phi < 1$),
and define
\bd
	k_1 = \phi , k_2 = \th + \phi , k_3 = \sqrt{r(\th^2+\phi^2)} .
\ed
The proof consists of showing that, under the stated hypotheses,
there exists a matrix $Y \in \Rno$ that satisfies the hypotheses of
Lemma \ref{lemma:31}.

We start with \eqref{eq:105a}.
	Lemma \ref{lemma:83} states the following:
If $Z\in\T$ and $\Zb:= (1/\al)\P_T E_\OM\circ Z-Z$,
then
\be\label{eq:301}
\nmF{\Zb} \leq (\theta + \phi) \nmF{Z}.
\ee
where $\th$ is defined in \eqref{eq:23} and $\phi$ is defined in \eqref{eq:80aa}.
Therefore \eqref{eq:105a} is satisfied with $k_2 = \th + \phi$.
Next, let $W_0 := UV^\top \in \T$, and define
\bd
Y = \foal E_\OM \circ W_0 .
\ed
Then
\bd
\PT (Y) - UV^\top = \foal \PT E_\OM \circ W_0 - W_0 .
\ed
Now we can write $W_0 = UB^\top$ where $B = V$, and apply Lemma
\ref{lemma:82}.
This gives
\bd
\nmF{\PT (Y) - W_0} \leq \sqrt{r(\th^2+\phi^2)} = k_3 .
\ed
Finally, observe that, because $W_0 \in \T$, we can reason as follows:
\begin{eqnarray*}
\nmS{ \PTP Y} & = & \nmS{\PTP(Y - W_0)} \\
& \leq & \nmS{Y - W_0} = \nmS{ M \circ W_0} .
\end{eqnarray*}
Now we can apply Lemma \ref{lemma:52}, with
\bd
a = \sqrt{ \frac{\mu_0 r}{n_r}} ,
b = \sqrt{ \frac{\mu_0 r}{n_c}} ,
\nmS{M} = \frac{\s_2}{\al} .
\ed
This gives
\bd
\nmS{\PTP Y} \leq \frac{\s_2}{\al} \frac{\mu_0 r}{\sqrt{n_r n_c}} = \phi .
\ed
Hence we can choose $k_1$.
The theorem now follows from applying Lemma \ref{lemma:31}
\end{proof}

There is considerable scope for improving Theorem \ref{thm:35}.
As shown earlier, we can take $k_2 = \th+\phi$.
Define $W_0 = UV^T$ as before, and define $W_i$ as
\be\label{eq:102}
W_i = W_{i-1} - (1/\al) \P_\T(E_\OM \circ W_{i-1}) 
\ee
\be\label{eq:103}
Y_p = \sum_{i=0}^{p-1} (1/\al)E_\OM \circ W_i
\ee
Then $E_\OM \circ Y_p = Y_p$ for every integer $p$.
To establish \eqref{eq:106}, note that
\bd
\foal \P_\T E_\OM \circ W_i = W_i - W_{i+1} .
\ed
So
\bd
\P_\T Y_p = \sum_{i=0}^{p-1} (W_i - W_{i+1}) = W_0 - W_p .
\ed
Therefore
\bd
\nmF{\P_\T Y_p - W_0} = \nmF{W_p} \leq (\th + \phi)^p \nmF{W_0} .
\ed
If $\th+\phi < 1$, then \eqref{eq:106} holds for $p$ sufficiently large,
no matter what $k_3$ is.
The difficulty however is to find a constant $k_1$ such that \eqref{eq:105}
is satisfied.
This is an object of ongoing research by the authors.

\section{Phase Transition Studies}\label{sec:Phase}

The bounds in Theorem \ref{thm:35} provide sufficient conditions
for matrix completion using nuclear norm minimization,
when the sample set  is chosen as the edge set of
a Ramanujan graph or a Ramanujan bigraph.
These results are only sufficient conditions.
It is possible to determine how close these sufficient conditions
are to being necessary via numerical simulations.
That is the objective of the present section.
Our simulations show that choosing $d \approx 3r$
seems to suffice to recover randomly generated matrices of rank $r$.
Thus $r_0 \approx d/3$ is the critical value below which the success
ratio is 100\% for recovering randomly generated matrices of rank $r \leq r_0$.
Moreover, the simulations show the presence of a ``phase transition,''
whereby the likelihood of success
falls sharply from 100\% to 0\% as $r$ is increased by just $2$ or $3$
above the critical value $r_0$.
These observations are consistent with previous numerical studies
on phase transition.
Our numerical results are presented next, and they are placed in context
against earlier results in the subsequent subsection.

\subsection{Numerical Experiments}

We carried out some numerical experiments
on the behavior of nuclear norm minimization for matrix completion,
on randomly generated low rank matrices.
To construct Ramanujan graphs to choose the sample matrices,
we used the so-called LPS construction proposed in
\cite{Lubotzky-et-al88}.
This construction is based on choosing two prime numbers $p, q$ both
congruent to $1 \mod 4$.
The resulting graph has $(q(q^2-1))/2$ vertices and degree $p+1$.
In the original construction, $p < q$.
However, it is possible to choose $p > q$ provided some consistency
conditions are satisfied.
The authors have written {\tt Matlab} code that implements the LPS
construction; this code is available from the authors upon request.

For illustrative purposes,
we chose $q = 13$, which leads to graphs with $n = 1,092$ vertices.
By varying the prime number $p$, we could generate
Ramanujan graphs with varying degrees.
Every prime number equal to $1 \mod 4$ between $5$ and $157$
results in a Ramanujan graph using the LPS construction.
However, still larger choices of $p$ are permissible, such as
$197, 229$ and $293$.
However, for some intermediate values of $p$ between $157$ and $197$,
the resulting LPS construction would have multiple edges between some
pairs of vertices.

For each choice of $p$, (that is, each choice of the degree $d$),
we varied the rank $r$ from 1 upwards, and for each $r$ generated
100 random matrices of dimensions $1092 \times 1092$ and rank $r$.
Then we applied nuclear norm minimization, and examined what fraction
of the 100 rank $r$ matrices were correctly recovered.
The criterion for correct recovery was that the normalized Frobenius
norm $\nmF{\Xh-X}/\nmF{X}$ was less than $10^{-6}$.
However, the results are quite insensitive to this number.
The objective was to
determine how the critical value of the rank $\bar{r}$, defined as the value 
below which 100\% recovery takes place,
depends on $d = p+1$.
The results are shown in Figure \ref{fig:72}.
It can be observed from this figure that
$\bar{r} \approx d/3 = (p+1)/3$ over the entire range of $p$.
For comparison purposes, we repeated the study, with the Ramanjuan graph
sample set replaced by randomly chosen samples.
There is very little difference between the two curves.
In a sense this is not surprising, because Ramanujan graphs replicate
many of the desirable properties of random graphs, including
expansion properties.
However, the Ramanujan construction provides a systematic method for
selecting the sample set.

More interesting was the fact that, if $r$ was increased by just one two
above the critical value $\bar{r}$,
the percentage of accurately recovered matrices
fell sharply from 100\% to 0\%.
This is illustrated in Figure \ref{fig:73} for $p = 197$ or $d = 198$,
in Figure \ref{fig:74} for $p = 229$ of $d = 230$,
and in Figure \ref{fig:75} for $p = 293$ or $d = 294$.
This phenomenon is known as ``phase transition,'' and is well-known
for vector recovery using $\ell_1$-norm minimization.
However, so far as we are aware, a similar phenomenon has not been reported
for matrix completion using nuclear norm minimization.

\bfig
\bc
\includegraphics[height=70mm]{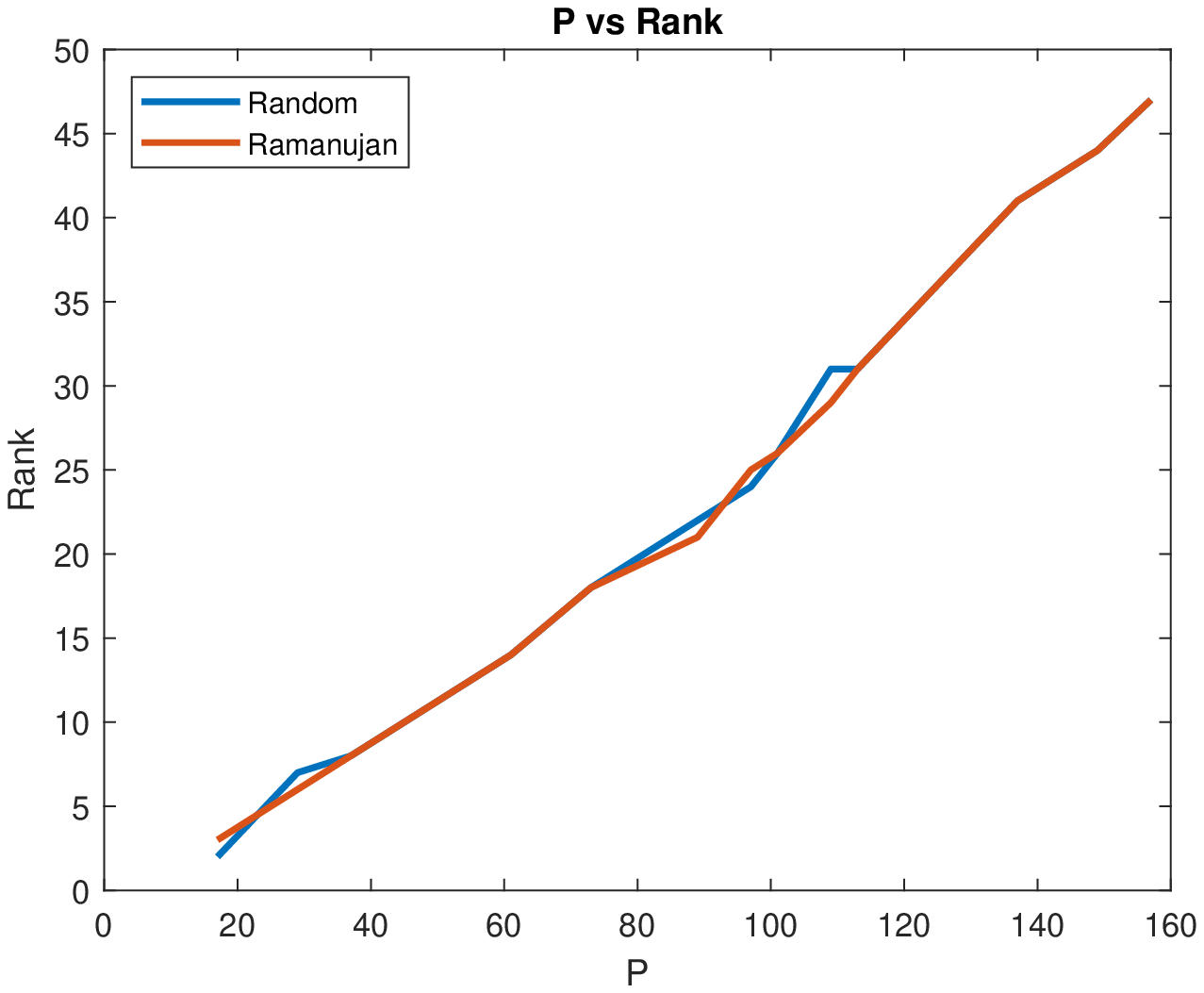}
\ec
\caption{Critical value of rank versus the degree of the LPS Ramanujan graph
with 1,092 vertices}
\label{fig:72}
\efig

\begin{table}
\bc
\btab{|c|c|c|}
\hline
$d$ & $\bar{r}$ & $\bar{r}/d$ \\
\hline
198 & 62 & 0.3147 \\
230 & 75 & 0.3275 \\
294 & 102 & 0.3481 \\
\hline
\etab
\ec
\caption{Degree vs.\ critical rank in high degree LPS Ramanujan graphs}
\end{table}

\bfig
\bc
\includegraphics[height=70mm]{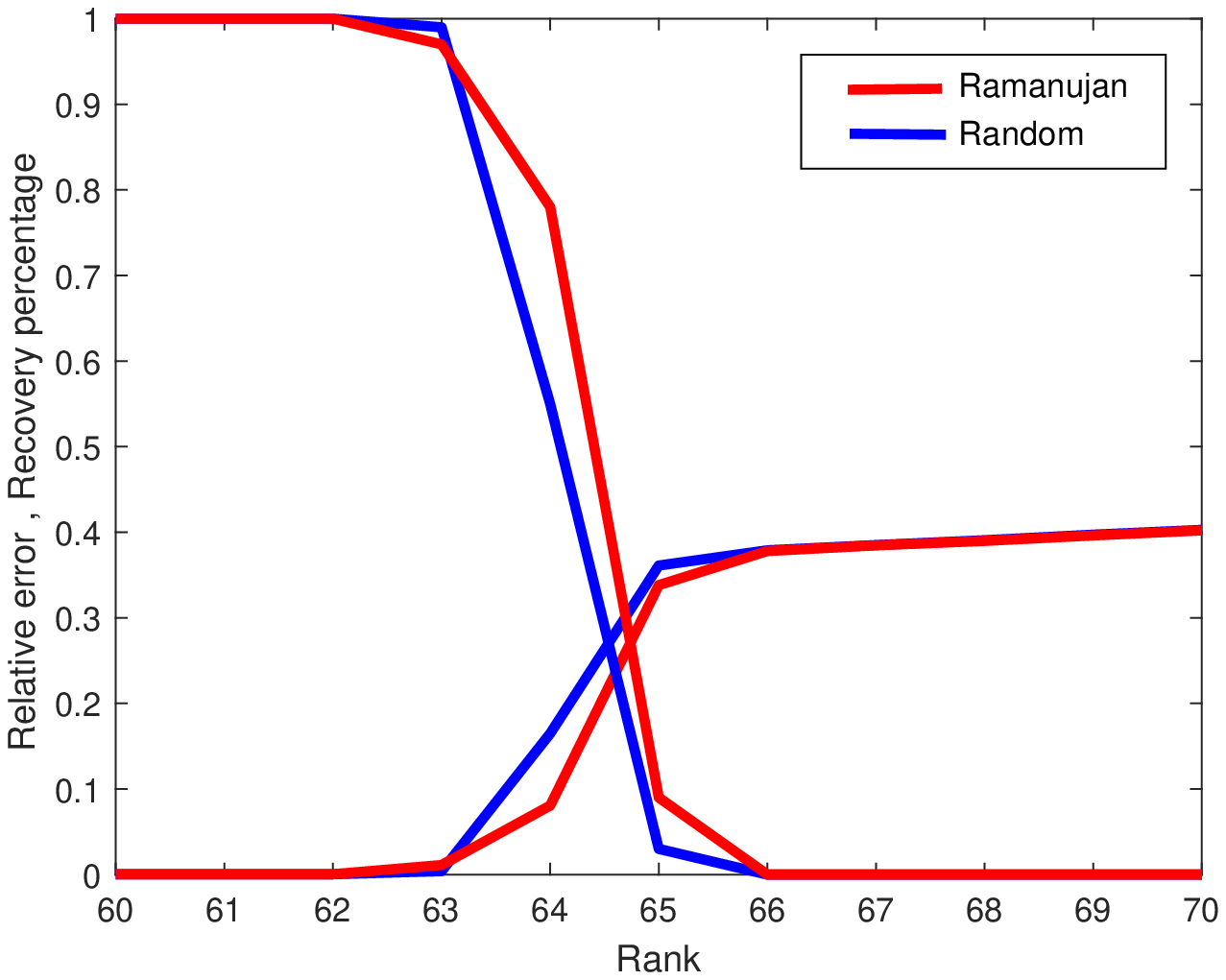}
\ec
\caption{Phase transition behavior in LPS graph
with 1,092 vertices with $p = 197$}
\label{fig:73}
\efig

\bfig
\bc
\includegraphics[height=70mm]{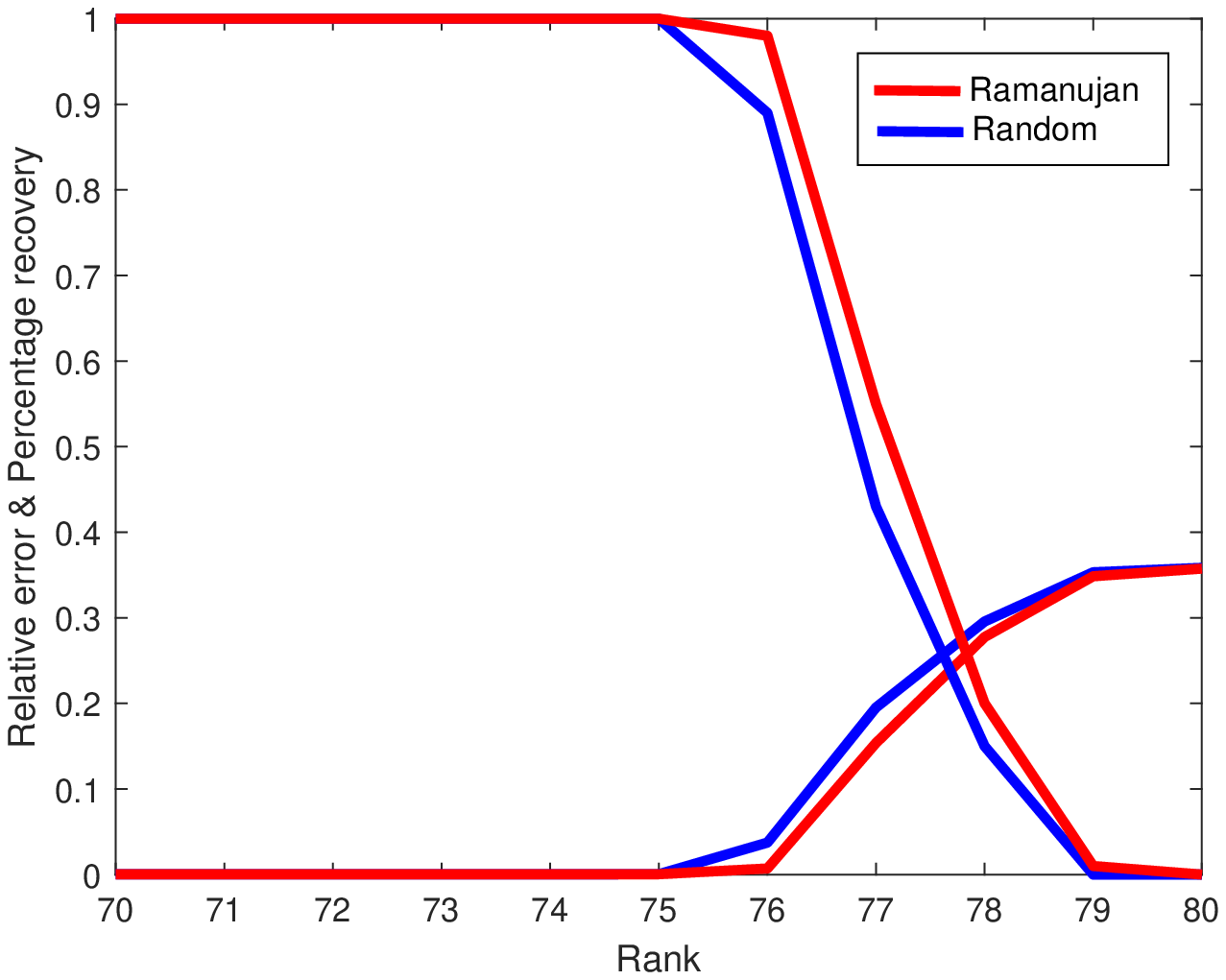}
\ec
\caption{Phase transition behavior in LPS graph
with 1,092 vertices with $p = 229$}
\label{fig:74}
\efig

\bfig
\bc
\includegraphics[height=70mm]{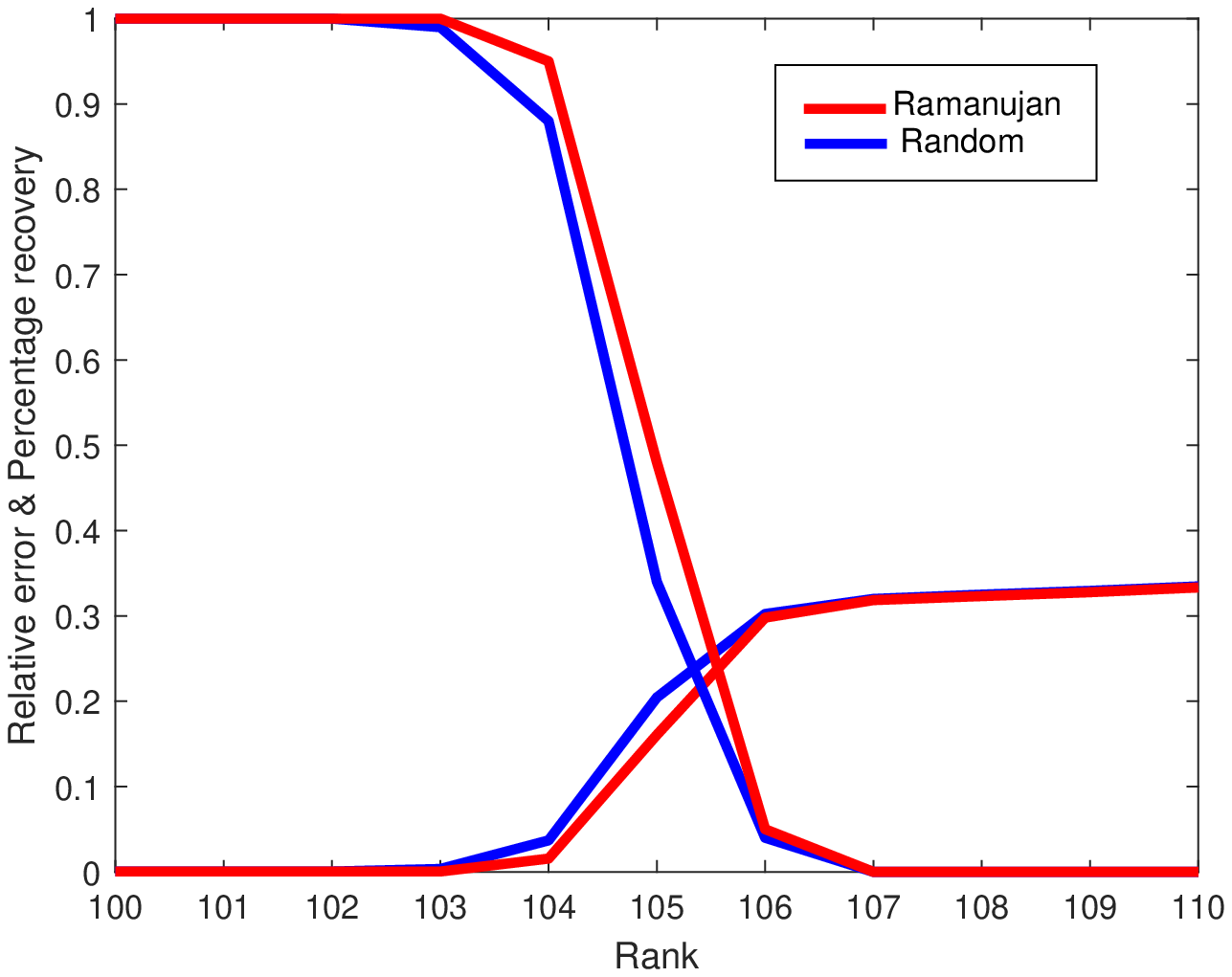}
\ec
\caption{Phase transition behavior in LPS graph
with 1,092 vertices with $p = 293$}
\label{fig:75}
\efig

\subsection{Comparison with Previous Studies}

The phenomenon of phase transition is well-known in the case of
vector recovery using $\ell_1$-norm minimization.
It is also well-understood, and is
discussed in several papers by David Donoho, such as
\cite{Donoho-Tanner-PNAS05,Donoho-et-al-PNAS09,Donoho-Tanner-JAMS09,
Donoho-Tanner-PTRSA09,Donoho-Tanner10,Donoho-et-al-ISIT12,Donoho-et-al-TIT13a,
Donoho-et-al-TIT13b}.
A very general theory of the behavior of convex optimization on
randomly generated data is given in \cite{Amelunxen-et-al14}.
A concept called the ``statistical dimension'' is introduced, and it is
established that convex optimization exhibits ``phase transition,''
whereby the success rate of the optimization algorithm goes from 100\%
to 0\% very quickly as the input to the algorithm is changed.
The width of the transition region is linear in the statistical dimension.
In contrast, there are relatively fewer papers that have studied
phase transition in matrix completion, and a proper theory is not
yet available.
In principle, the theory in \cite{Amelunxen-et-al14}
applies also to matrix completion using nuclear norm minimization.
However, it is not easy to work out the appropriate statistical dimension
in this case.
In \cite{P-ABN16}, the behaviour of the manifold approach to
recovering randomly generated matrices is studied.
A related paper is \cite{Donoho-et-al-PNAS13}, which studies matrix
\textit{recovery}, and not matrix \textit{completion}.
In that paper, the measurements consist of taking the Frobenius inner product
of the unknown matrix with randomly generated Gaussian matrices.

Now we describe some of the prior work on phase transition in matrix
completion via nuclear norm minimization.
In all the studies described below, the unknown matrix $X$
is assumed to be square, so we use the neutral symbol $n = n_c = n_c$
to denote its size.
Throughout, $r$ denotes its rank, and $m$ denotes the number of
measurements, that is, $|\OM|$.
This is consistent with the notation in the present paper, though
the notation in the papers discussed below is in general different.
Note that in some papers, $X$ satisfies additional constraints such as
symmetry, positive semidefiniteness, Hankel or Toeplitz structure, etc.
In \cite{Candes-Recht08}, which first introduced this approach,
there are some examples based on randomly generated low-rank matrices
with $n = 40$ or $n = 50$.
In \cite{KMO10a}, $n = 500$, and $r \in \{ 10, 20, 40 \}$.
In \cite{Chen-Chi-TIT14}, $n = 11$ or $n = 15$.
In \cite{Chen-Chi-TIT15}, $n = 50$.
In \cite{ZHWC-STSP18}, $n = 300$.
Finally, in \cite{YFS-Ent18}, $n = 1,000$, and $r/n \in [0.04,0.4]$.
In all cases, it is observed that the critical rank $r_0$ at which
the recovery ratio drops sharply from 100\% to 0\% is essentially
linear in $m$, the number of measurements.
In our setting, since $m = dn$ where $d$ is the degree of the Ramanujan
graph, this translates to $r_0$ being linear in $d$, as observed in the
next subsection.
Thus, to summarize, available numerical studies (including ours),
indicate the presence of a phase transition, and a linear relationship
between the critical rank $r_0$ and the number of measurements $m$.

\section{Conclusions and Future Work}\label{sec:Conc}

In this paper we have studied the matrix completion problem with emphasis
on choosing the elements to be sampled in a deterministic fashion.
We do this by choosing the sample matrix to equal the biadjacency matrix
of an asymmetric Ramanujan graph, or Ramanujan bigraph.
We have derived a sufficient
condition that guarantees \textit{exact} recovery of the unknown matrix
with noise-free measurements using nuclear norm minimization,
and \textit{stable} recovery under non-sparse noisy measurements.
We believe that we are presenting the very first correct result on
exact recovery using nuclear norm minimization and a deterministic sampling
pattern.
An earlier paper \cite[Theorem 4.2]{Bhoj-Jain14} claims a similar result,
but there is one step in the proof that we believe is not justified.
This is elaborated in the Appendix.

The sufficient condition given here is very conservative.
It requires that the degree of the Ramanujan graph should be $\OM(r^3)$
where $r$ is the rank of the matrix to be recovered.
Turning this around, our result implies that given a Ramanujan graph of
degree $d$, we can guarantee exact recovery only when $r = O(d^{1/3})$.
This naturally raises the question as to
how close is the sufficient condition derived here to being necessary?
Our numerical simulations have shown that $d \approx 3r$
seems to be sufficient to recover randomly generated matrices of rank $r$.
More interesting, if $r$ is increased by just two or three above
this critical value of $d/3$, the percentage of the randomly
generated matrices that are recovered falls from 100\% to 0\%,
a phenomenon known as ``phase transition.''

One of the advantages of the random sampling approach is that it is
relatively easy to account for ``missing'' measurements, by ensuring
that the missing location is never sampled.
In the present approach, this leads to a very interesting problem in graph
theory, namely the construction of Ramanujan graphs when there is a ``bar''
on having edges between specified pairs of vertices.
The authors are exploring this question, which would be of considerable
interest to graph theorists, quite apart from the matrix completion
researchers.

\section*{Appendix}

In this appendix we point out an error in the proof of
\cite[Theorem 4.2]{Bhoj-Jain14}, which
is based on a recursion lemma \cite[Lemma 7.3]{Bhoj-Jain14}.
However, the proof of the recursion lemma is not given in \cite{Bhoj-Jain14},
but is found in the longer version of the paper in \cite{Bhoj-Jain-arxiv14}.
On page 19 of \cite{Bhoj-Jain-arxiv14}, the third displayed (but unnumbered)
equation is as follows:\footnote{For the convenience of the reader, we use the
notation in \cite{Bhoj-Jain-arxiv14}, as opposed to the current notation.
Also, while the equation in question
is shown on a single line in \cite{Bhoj-Jain-arxiv14},
it is broken into several lines here to fit into the double-column format.}
\beq\label{eq:ap02}
\sum_{k = 1}^{r} \nmeusq{\uh.U_k} 
& = &  \sum_{k = 1}^{r} \sum_{l = 1}^{n} \IP{\Up^i}{\Up^l}^2 U_{lk}^2 \nonumber\\
& = & \sum_{l = 1}^{n} \IP{\Up^i}{\Up^l}^2 \nmeusq{U^l} \nonumber\\
& =^{(a)} & \sum_{l = 1}^{n} \IP{U^i}{U^l}^2 \nmeusq{U^l} \nonumber\\
& \leq & (\mu_0 r/n)^2,
\eeq
where $\uh = \Up\Up^i$. 

Let us focus on the claimed equality highlighted by us as $(a)$.
This equation is incorrect.
If we represent $\Up\Upt$ as $ (I_{n_r}-U\Ut)$, then by orthogonality we have
\bd
\IP{U^i}{U^j} + \IP{\Up^i}{\Up^j} = 0 \fa i \neq j ,
\ed
which in turn implies that
\be\label{eq:ap01}
|\IP{U^i}{U^j}| = |\IP{\Up^i}{\Up^j}| \ \forall i\neq j .
\ee
However, when $i = j$, we have that
\bd
\nmeusq{U^i}+\nmeusq{\Up^i} = 1 \imp \IP{\Up^i}{\Up^i} = 1 - \IP{U^i}{U^i }
\ed
for all $i$.
Therefore 
\bd
\nmeusq{\Up^i} = \nmeusq{U^i} \imp \nmeusq{\Up^i} = \nmeusq{U^i} = 1/2 .
\ed
More elaborately
\begin{eqnarray*}
\sum_{l = 1}^{n} \IP{\Up^i}{\Up^l}^2 \nmeusq{U^l}
& = & \nmeu{\Up^i}^4 \nmeusq{U^i} \\
& + & \sum_{l \neq i} \IP{\Up^i}{\Up^l}^2 \nmeusq{U^l} \\
& = & (1 - \nmeusq{U^i})^2 \nmeusq{U^i} \\
& + & \sum_{l \neq i} \IP{U^i}{U^l}^2 \nmeusq{U^l} \\
& = &  \sum_{l = 1}^{n} \IP{U^i}{U^l}^2 \nmeusq{U^l} \\
& + & (1 - 2 \nmeusq{U^i}) \nmeusq{U^i} .
\end{eqnarray*}
Therefore, unless $\nmeusq{U^i} = 1/2$ or zero, the step highlighted as $(a)$
in \eqref{eq:ap02} is false for that value of $i$.
The conclusion is that, in order for \eqref{eq:ap02} to hold,
\textit{every} row of $U$ must either be identically zero, or have
Euclidean norm of $1/2$.
A row of $U$ being identically zero makes the corresponding row of the
unknown matrix $X$ also identically zero.
In short, \eqref{eq:ap01} cannot hold except under extremely restrictive
and unrealistic conditions.
Similar reasoning is used for $V$ which is also not correct.
Hence using this Lemma 7.3 in bounding $\nmS{\PTP(Y)}$ on page 12 
of \cite{Bhoj-Jain-arxiv14} makes
the proof of the Theorem 4.2 incorrect.
However, it is of course possible that the theorem itself is correct.

\section*{Acknowledgement}

The authors thank the Associate Editor and three anonymous reviewers for
drawing their attention to several relevant references, and for their
constructive criticisms that have greatly enhanced the paper.
In addition,
the authors thank Prof.\ Alex Lubotzky of Hebrew University, Jerusalem,
Israel for useful hints on the construction of Ramanujan graphs of high degree,
and Prof.\ Cristina Ballantine of the College of the Holy Cross for
discussions on the construction of Ramanujan bigraphs.

\bibliographystyle{IEEEtran}

\bibliography{Comp-Sens}

\end{document}